\documentclass{article}

\PassOptionsToPackage{numbers, compress}{natbib}

\usepackage[preprint]{nips_2018}

\usepackage{amsfonts}
\usepackage{amsmath}
\usepackage{commath}
\usepackage{mathbbol}
\usepackage{cancel}
\usepackage{multirow, pbox}

\usepackage{times}
\usepackage{graphicx} 
\usepackage{subfigure} 

\usepackage{natbib}

\usepackage{algorithm}
\usepackage{algorithmic}

\usepackage{amsthm,amssymb}
\newtheorem{assumption}{Assumption}

\newtheorem{proposition}{Proposition}
\newtheorem{lemma}{Lemma}
\newtheorem{theorem}{Theorem}

\newcommand{\X}{\mathcal{X}}

\title{A Lyapunov-based Approach to Safe Reinforcement Learning}

\author{
    Yinlam Chow \\
  Google DeepMind \\
  \texttt{yinlamchow@google.com} \\
  \AND
  Ofir Nachum \\ 
  Google Brain \\
    \texttt{ofirnachum@google.com} \\
  \AND
    Edgar Duenez-Guzman \\
  Google DeepMind  \\ 
    \texttt{duenez@google.com} \\
   \AND
  Mohammad Ghavamzadeh \\
  Google DeepMind \\
    \texttt{ghavamza@google.com} \\
}

\begin{document}

\maketitle

\vspace{-0.3in}
\begin{abstract}
\vspace{-0.1in}
In many real-world reinforcement learning (RL) problems, besides optimizing the main objective function, an agent must concurrently avoid violating a number of constraints.
In particular, besides optimizing performance it is crucial to guarantee the \emph{safety} of an agent during training as well as deployment (e.g. a robot should avoid taking actions - exploratory or not - which irrevocably harm its hardware).
To incorporate safety in RL, we derive algorithms under the framework of constrained Markov decision problems (CMDPs), an extension of the standard Markov decision problems (MDPs) augmented with constraints on expected cumulative costs.
Our approach hinges on a novel \emph{Lyapunov} method.  
We define and present a method for constructing Lyapunov functions, which provide an effective way to guarantee the global safety of a behavior policy during training via a set of local, linear constraints.
Leveraging these theoretical underpinnings, we show how to use the Lyapunov approach to systematically transform dynamic programming (DP) and RL algorithms into their safe counterparts. 
To illustrate their effectiveness, we evaluate these algorithms in several CMDP planning and decision-making tasks on a safety benchmark domain. Our results show that our proposed method significantly outperforms existing baselines in balancing constraint satisfaction and performance.
\end{abstract}

\vspace{-0.2in}
\section{Introduction}
\vspace{-0.1in}

Reinforcement learning (RL) has shown exceptional successes in a variety of domains such as video games \citep{mnih2013playing} and recommender systems \citep{shani2005mdp}, where the main goal is to optimize a single return. However, in many real-world problems, besides optimizing the main objective (the return), there can exist several conflicting constraints that make RL challenging. In particular,  besides optimizing performance it is crucial to guarantee the \emph{safety} of an agent in deployment \citep{amodei2016concrete} as well as during training \citep{achiam2017constrained}.
For example, a robot agent should avoid taking actions which irrevocably harm its hardware; a recommender system must avoid presenting  harmful or offending items to users.

Sequential decision-making in non-deterministic
environments has been extensively studied in the literature under the framework of Markov decision problems (MDPs). 
To incorporate safety into the RL process, we are particularly interested in deriving algorithms under the context of constrained
Markov decision problems (CMDPs), which is an extension of MDPs with expected cumulative constraint costs. The additional constraint component of CMDPs increases flexibility in modeling problems with \emph{trajectory-based} constraints, when compared with other approaches that customize immediate costs in MDPs to handle constraints \citep{regan2009regret}. As shown in numerous applications from robot motion planning \citep{ono2015chance, moldovan2012safe, chow2015trading}, resource allocation \citep{mastronarde2011fast, junges2016safety}, and financial engineering \citep{abe2010optimizing, tamar2012policy}, it is more natural to define safety over the whole trajectory, instead of over particular state and action pairs. Under this framework, we denote an agent's behavior policy to be \emph{safe} if it satisfies the cumulative cost constraints of the CMDP.

Despite the capabilities of CMDPs, they have not been very popular in RL. One main reason is that, although optimal policies of finite CMDPs are Markov and stationary, and with known models the CMDP can be solved using linear programming (LP) \citep{altman1999constrained}, it is unclear how to extend this algorithm to handle cases when the model is unknown, or when the state and action spaces are large or continuous. A well-known approach to solve CMDPs is the Lagrangian method \citep{altman1998constrained, geibel2005risk}, which augments the standard expected reward objective with a penalty on constraint violation. 
With a fixed Lagrange multiplier, one can use standard dynamic programming (DP) or RL algorithms to solve for an optimal policy. 
With a learnable Lagrange multiplier, one must solve the resulting saddle point problem. 
However, several studies \citep{lee2017first} showed that iteratively solving the saddle point is apt to run into numerical stability issues. More importantly, the Lagrangian policy is only safe \emph{asymptotically} and makes little guarantee with regards to safety of the behavior policy during each training iteration.

Motivated by these observations, several recent works have derived surrogate algorithms for solving CMDPs, which transform the original constraint to a more conservative one that yields an easier problem to solve.
A straight-forward approach is to replace the cumulative constraint cost with a conservative \emph{stepwise} surrogate constraint \citep{el2016convex} that only depends on current state and action pair. Since this surrogate constraint can be easily embedded into the admissible control set, this formulation can be modeled by an MDP that has a restricted set of admissible actions. Another surrogate algorithm was proposed by \cite{gabor1998multi}, in which the algorithm first computes a uniform \emph{super-martingale} constraint value function surrogate w.r.t. all policies. 
Then one computes a CMDP feasible policy by optimizing the surrogate problem using the \emph{lexicographical ordering} method \citep{schmitt2006complexity}.
These methods are advantageous in the sense that (i) there are RL algorithms available to handle the surrogate problems (for example see \cite{dalal2018safe} for the step-wise surrogate, and see \cite{mossalam2016multi} for the super-martingale surrogate), (ii) the policy returned by this method is safe, even during training. 
However, the main drawback of these approaches is their conservativeness. Characterizing sub-optimality performance of the corresponding solution policy also remains a challenging task. 
On the other hand, recently in policy gradient, \cite{achiam2017constrained} proposed the \emph{constrained policy optimization} (CPO) method that extends trust-region policy optimization (TRPO) to handle the CMDP constraints. While this algorithm is scalable, and its policy is safe during training, analyzing its convergence is challenging, and applying this methodology to other RL algorithms is non-trivial.   

 Lyapunov functions have been extensively used in control theory to analyze the stability of dynamic systems~\citep{khalil1996noninear, neely2010stochastic}.
A Lyapunov function is a type of scalar potential function that keeps track of the \emph{energy} that a system continually dissipates. Besides modeling physical energy, Lyapunov functions can also represent abstract quantities, such as the steady-state performance of a Markov process \citep{glynn2008bounding}. In many fields, Lyapunov functions provide a powerful paradigm to translate global properties of a system to local ones, and vice-versa.
Using Lyapunov functions in RL was first studied by \cite{perkins2002lyapunov}, where Lyapunov functions were used to guarantee closed-loop stability of an agent.
Recently \cite{berkenkamp2017safe} used Lyapunov functions to guarantee a model-based RL agent's ability to re-enter an ``attraction region'' during exploration. However no previous works have used Lyapunov approaches to explicitly model constraints in a CMDP.  
Furthermore, one major drawback of these approaches is that the Lyapunov functions are hand-crafted, and there are no principled guidelines on designing Lyapunov functions that can guarantee an agent's performance.

The contribution of this paper is four-fold. First we formulate the problem of safe reinforcement learning as a CMDP and propose a novel \emph{Lyapunov approach} for solving them. 
While the main challenge of other Lyapunov-based methods is to design a Lyapunov function candidate, we propose an LP-based algorithm to construct Lyapunov functions w.r.t. generic CMDP constraints. We also show that our method is guaranteed to always return a feasible policy and, under certain technical assumptions, it achieves optimality. 
Second, leveraging the theoretical underpinnings of the Lyapunov approach, we present two safe DP algorithms -- safe policy iteration (SPI) and safe value iteration (SVI) -- and analyze the feasibility and performance of these algorithms. 
Third, to handle unknown environment models and large state/action spaces, we develop two scalable, safe RL algorithms -- (i) \emph{safe DQN}, an off-policy fitted $Q-$iteration method, and (ii) \emph{safe DPI}, an approximate policy iteration method.
Fourth, to illustrate the effectiveness of these algorithms, we evaluate them in several tasks on a benchmark 2D planning problem, and show that they outperform common baselines in terms of balancing performance and constraint satisfaction.


\vspace{-0.1in}
\section{Preliminaries}\label{sec:prelim}
\vspace{-0.1in}

We consider the RL problem in which the agent's interaction with the system is modeled as a Markov decision process (MDP). 
A MDP is a tuple $(\mathcal X,\mathcal A,c,P,x_0)$, where $\mathcal X=\mathcal X'\cup\{x_{\text{Term}}\}$ is the state space, with transient state space $\mathcal X'$ and terminal state $x_{\text{Term}}$; $\mathcal A$ is the action space; $c(x,a)\in[0,C_{\max}]$ is the immediate cost function (negative reward);
$P(\cdot|x,a)$ is the transition probability distribution; and $x_0\in \mathcal X'$ is the initial state. 
Our results easily generalize to random initial states and random costs, but for simplicity we will focus on the case of deterministic initial state and immediate cost. 
In a more general setting where cumulative constraints are taken into account, we define a constrained Markov decision process (CMDP), which extends the MDP 
model by introducing additional costs and associated constraints. A CMDP is defined by  $(\mathcal X,\mathcal A,c,d,P,x_0,d_0)$,
where the components $\mathcal X,\mathcal A,c,P,x_0$ are the same for the unconstrained MDP; $d(x)\in[0,D_{\max}]$ is the immediate constraint cost; and $d_0 \in \mathbb{R}_{\geq 0}$ is an upper bound for the expected cumulative (through time) constraint cost.
To formalize the optimization problem associated with CMDPs, let $\Delta$ be the set of Markov stationary policies, i.e., $\Delta(x)=\{\pi(\cdot|x):\mathcal X\rightarrow \mathbb R_{\geq 0s}:\sum_{a}\pi(a|x)=1\}$ for any state $x\in\mathcal X$.
Also let $\mathrm{T}^*$ be a random variable corresponding to the first-hitting time of the terminal state $x_{\text{Term}}$ induced by policy $\pi$.  In this paper, we follow the standard notion of transient MDPs and assume that the first-hitting time is uniformly bounded by an upper bound $\overline{\mathrm{T}}$.
This assumption can be justified by the fact that sample trajectories collected in most RL algorithms consist of a finite
stopping time (also known as a time-out); the assumption may also be relaxed in cases where a discount $\gamma<1$ is applied on future costs.
For notational convenience, at each state $x\in\mathcal X'$, we define the generic Bellman operator w.r.t. policy $\pi\in\Delta$ and generic cost function $h$:
$
T_{\pi,h}[V](x)=\sum_{a}\pi(a|x)\!\left[h(x,a)\!+\!\sum_{x'\in\mathcal X'}\!\!P(x'|x,a)V(x')\right].
$

Given a policy $\pi\in \Delta$, an initial state $x_0$, the cost function is defined as
$
\mathcal C_\pi(x_0):=\mathbb E\big[\sum_{t=0}^{\mathrm{T}^*-1}c(x_t, a_t)\mid x_0,\pi\big],
$
and the safety constraint is defined as
$
\mathcal D_\pi(x_0)\leq d_0
$
where the safety constraint function is given by
$
\mathcal D_\pi(x_0):=\mathbb E\big[\sum_{t=0}^{\mathrm{T}^*-1}d(x_t)\mid x_0,\pi\big].
$
In general the CMDP problem we wish to solve is given as follows:
\begin{quote} {\bf Problem $\mathcal{OPT}$:} Given an initial state $x_0$ and a threshold $d_0$, solve $\min_{\pi \in \Delta} \,\left\{\mathcal C_\pi(x_0): \mathcal D_\pi(x_0)\leq d_0\right\}.$
If there is a non-empty solution, the optimal policy is denoted by $\pi^*$.
\end{quote}
Under the transient CMDP assumption, Theorem 8.1 in \cite{altman1999constrained}, shows that if the feasibility set is non-empty, then there exists an optimal policy in the class of stationary Markovian policies $\Delta$. To motivate the CMDP formulation studied in this paper, in Appendix \ref{sec:safety_problem_examples}, we include two real-life examples in modeling safety using (i) the reachability constraint, and (ii) the constraint that limits the agent's visits to undesirable states. Recently there has been a number of works on CMDP  algorithms; their details can be found in Appendix \ref{sec:existing approaches}. 

\vspace{-0.1in}
\section{A Lyapunov Approach for Solving CMDPs}\label{sec:lyapunov_cmdp}
\vspace{-0.1in}


In this section we develop a novel methodology for solving CMDPs using the \emph{Lyapunov approach}.
To start with, without loss of generality assume we have access to a \emph{baseline} feasible policy of problem $\mathcal{OPT}$, namely $\pi_B\in\Delta$\footnote{One example of $\pi_B$ is a policy that minimizes the constraint, i.e., $\pi_B(\cdot|x)\in\arg\min_{\pi\in\Delta(x)}\mathcal D_{\pi}(x)$.}. We define a non-empty\footnote{To see this, the constraint cost function $\mathcal D_{\pi_{B}}(x)$, is a valid Lyapunov function, i.e., $\mathcal D_{\pi_{B}}(x_0)\leq d_0$, $\mathcal D_{\pi_{B}}(x)=0$, for $x\in\mathcal X\setminus\mathcal X'$, and 
$
\mathcal D_{\pi_{B}}(x)=T_{\pi_{B},d}[\mathcal D_{\pi_{B}}](x)=\mathbb E\left[\sum_{t=0}^{\mathrm{T}^*-1}d(x_t)\mid \pi_{B},x\right],\,\,\forall x\in\mathcal X'.
$}
 set of {Lyapunov functions} w.r.t. initial state $x_0\in\mathcal X$ and constraint threshold $d_0$ as
\begin{equation}\label{eq:Lyap_fn}
\mathcal L_{\pi_B}(x_0,d_0)\!=\!\bigg\{\!L\!:\mathcal X\!\rightarrow\!\mathbb R_{\geq 0}\!: \!T_{\pi_{B},d}[L](x)\!\leq\! L(x),\forall x\in\mathcal X';\, L(x)=0,\,\forall x\in\mathcal X\setminus\mathcal X';\,L(x_0)\leq d_0 \bigg\}.
\end{equation}
For any arbitrary Lyapunov function $L\in\mathcal L_{\pi_{B}}(x_0,d_0)$, denote by $\mathcal F_{L}(x)=\left\{\pi(\cdot|x)\in\Delta:T_{\pi,d}[L](x)\leq\! L(x)\right\}$ the set of $L-$induced Markov stationary policies. 
Since $T_{\pi,d}$ is a contraction mapping \citep{bertsekas1995dynamic}, clearly any $L-$induced policy $\pi$ has the following property: $\mathcal D_{\pi}(x)=\lim_{k\rightarrow\infty}T^k_{\pi,d}[L](x)\leq L(x)$, $\forall x\in\mathcal X'$. 
Together with the property of $ L(x_0)\leq d_0$, this further implies any $L-$induced policy is a feasible policy of problem $\mathcal{OPT}$. 
However in general the set $\mathcal F_{L}(x)$ does not necessarily contain any optimal policies of problem $\mathcal{OPT}$, and our main contribution is to design a Lyapunov function (w.r.t. a baseline policy) that provides this guarantee. 
In other words, our main goal is to construct a Lyapunov function $L\in \mathcal L_{\pi_B}(x_0,d_0)$ such that
\begin{equation}
L(x)\geq T_{\pi^*,d}[L](x),\,\,L(x_0)\leq d_0.\label{eq:opt_lyap}
\end{equation}

Before getting into the main results, we consider the following important technical lemma, which states that 
 with appropriate \emph{cost-shaping}, one can always transform the constraint value function $\mathcal D_{\pi^*}(x)$ w.r.t. optimal policy $\pi^*$ into a Lyapunov function that is induced by $\pi_B$, i.e., $L_\epsilon(x)\in \mathcal L_{\pi_B}(x_0,d_0)$. The proof of this lemma can be found in Appendix \ref{appendix:lem:existence_L}.
\begin{lemma}\label{lem:existence_L}
There exists an auxiliary constraint cost $\epsilon:\mathcal X'\rightarrow\mathbb R$ such that the Lyapunov function is given by $L_{\epsilon}(x)=\mathbb E\left[\sum_{t=0}^{\mathrm{T}^*-1}d(x_t)+\epsilon(x_t)\mid \pi_{B},x\right],\forall x\in\mathcal X',
$
and $L_{\epsilon}(x)=0$ for $x\in\mathcal X\setminus\mathcal X'$.
Moreover, $L_{\epsilon}$ is equal to the constraint value function w.r.t. $\pi^*$, i.e., $L_{\epsilon}(x)=\mathcal D_{\pi^*}(x)$. 
\end{lemma}
From the structure of $L_{\epsilon}$, one can see that the auxiliary constraint cost function $\epsilon$ is uniformly bounded by $\epsilon^*(x):=2 \overline{\mathrm{T}} D_{\max}D_{TV}(\pi^*||\pi_{B})(x)$,\footnote{The definition of total variation distance is given by $D_{TV}(\pi^*||\pi_{B})(x)=\frac{1}{2}\sum_{a\in\mathcal A}|\pi_B(a|x)-\pi^*(a|x)|$.} i.e., $\epsilon(x)\in[-\epsilon^*(x),\epsilon^*(x)]$, for any $x\in\mathcal X'$. However in general it is unclear how to construct such a cost-shaping term $\epsilon$ without explicitly knowing $\pi^*$ a-priori. 
Rather, inspired by this result, we consider the bound $\epsilon^*$ to propose a Lyapunov function \emph{candidate} $L_{\epsilon^*}$. Immediately from its definition, this function has the following properties:
\begin{align}
&L_{\epsilon^*}(x)\geq T_{\pi_{B},d}[L_{\epsilon^*}](x),\,\,L_{\epsilon^*}(x)\geq \max\left\{\mathcal D_{\pi^*}(x),\mathcal D_{\pi_{B}}(x)\right\}\geq 0,\,\forall x\in\mathcal X'.\label{eq:baseline_lyap}
\end{align}
The first property is due to the facts that: (i) $\epsilon^*$ is a non-negative cost function; (ii) $T_{\pi_{B},d+\epsilon^*}$ is a contraction mapping, which by the fixed point theorem \citep{bertsekas1995dynamic} implies
$
L_{\epsilon^*}(x)=T_{\pi_{B},d+\epsilon^*}[L_{\epsilon^*}](x)\geq T_{\pi_{B},d}[L_{\epsilon^*}](x),\,\forall x\in\mathcal X'.
$
For the second property, from the above inequality one concludes that the Lyapunov function $L_{\epsilon^*}$ is a uniform upper bound to the constraint cost, i.e.,
$L_{\epsilon^*}(x)\geq \mathcal D_{\pi_{B}}(x)$,
because the constraint cost $\mathcal D_{\pi_{B}}(x)$ w.r.t. policy $\pi_B$ is the unique solution to the fixed-point equation $T_{\pi_{B},d}[V](x)=V(x)$, $x\in\mathcal X'$. 
On the other hand, by construction $\epsilon^*(x)$ is an upper-bound of the cost-shaping term $\epsilon(x)$. Therefore Lemma \ref{lem:existence_L} implies that Lyapunov function $L_{\epsilon^*}$ is a uniform upper bound to the constraint cost w.r.t. optimal policy $\pi^*$, i.e.,
$L_{\epsilon^*}(x)\geq \mathcal D_{\pi^*}(x)$.

To show that $L_{\epsilon^*}$ is a Lyapunov function that satisfies \eqref{eq:opt_lyap}, 
we propose the following condition that enforces a baseline policy $\pi_B$ to be \emph{sufficiently close} to an optimal policy $\pi^*$.
\begin{assumption}\label{assumption:pi_b}
The feasible baseline policy $\pi_{B}$ satisfies the following condition:
$
\max_{x\in\mathcal X'}\epsilon^*(x)\!\leq\!D_{\max}\cdot\min\left\{\frac{d_0-\mathcal D_{\pi_{B}}(x_0)}{\overline{\mathrm{T}} D_{\max}},\frac{\overline{\mathrm{T}}D_{\max}-\overline{\mathcal D}}{\overline{\mathrm{T}}D_{\max}+\overline{\mathcal D}}\right\},
$
where $\overline{\mathcal D}=\max_{x\in\mathcal X'}\max_{\pi} \mathcal D_\pi(x)$.
\end{assumption}

This condition characterizes the maximum allowable distance between $\pi_B$ and $\pi^*$, such that the set of $L_{\epsilon^*}-$induced policies contains an optimal policy. To formalize this claim, we have the following main result showing that $L_{\epsilon^*}\in\mathcal L_{\pi_{B}}(x_0,d_0)$, and the set of policies $\mathcal F_{L_{\epsilon^*}}$ contains an optimal policy.
\begin{theorem}\label{thm:L_star}
Suppose the baseline policy $\pi_B$ satisfies Assumption $1$, then on top of the properties in \eqref{eq:baseline_lyap}, the Lyapunov function candidate $L_{\epsilon^*}$ also satisfies the properties in \eqref{eq:opt_lyap}, and therefore its induced feasible set of policies $\mathcal F_{L_{\epsilon^*}}$ contains an optimal policy.
\end{theorem}
The proof of this theorem is given in Appendix \ref{appendix:thm:L_star}.  Suppose the distance between the baseline policy and the optimal policy can be estimated effectively. Using the above result, one can immediately determine if the set of $L_{\epsilon^*}-$induced policies contain an optimal policy.
 Equipped with the set of $L_{\epsilon^*}-$induced feasible policies, consider the following \emph{safe} Bellman operator:
\begin{equation}\label{eq:Bellman_lyap}
T[V](x)=\left\{\begin{array}{cl}
\min_{\pi\in\mathcal F_{L_{\epsilon^*}}(x)}T_{\pi,c}[V](x)&\text{if $x\in\mathcal X'$}\\
0&\text{otherwise}
\end{array}\right..
\end{equation}
Using standard analysis of Bellman operators, one can show that $T$ is a monotonic and contraction operator (see Appendix \ref{sec:appendix:safety_aware_bellman} for proof). This further implies that the solution of the fixed point equation $T[V](x)=V(x)$, $\forall x\in\mathcal X$, is unique. Let $V^*$ be such a value function. The following theorem shows that under Assumption \ref{assumption:pi_b}, $V^*(x_0)$ is a solution to problem $\mathcal{OPT}$. 
\begin{theorem}\label{thm:safe_opt}
Suppose the baseline policy $\pi_{B}$ satisfies Assumption \ref{assumption:pi_b}. Then, the fixed-point solution at $x=x_0$, i.e., $V^*(x_0)$, is equal to the solution of problem $\mathcal{OPT}$. Furthermore, an optimal policy can be constructed by $\pi^*(\cdot|x)\!\in\!\arg\min_{\pi\in\mathcal F_{L_{\epsilon^*}}(x)}T_{\pi,c}[V^*](x)$, $\forall x\!\in\!\mathcal X'$.
\end{theorem}
The proof of this theorem can be found in Appendix \ref{appendix:thm:safe_opt}.
This shows that under Assumption \ref{assumption:pi_b} an optimal policy of problem $\mathcal{OPT}$ can be solved using standard DP algorithms. Notice that verifying whether $\pi_B$ satisfies this assumption is still challenging because one requires a good estimate of $D_{TV}(\pi^*||\pi_{B})$. Yet to the best of our knowledge, this is \emph{the first result} that connects the optimality of CMDP to Bellman's principle of optimality. 
Another key observation is that, in practice we will explore ways of approximating $\epsilon^*$ via \emph{bootstrapping}, and empirically show that this approach achieves good performance while guaranteeing safety at each iteration. In particular, in the next section we will illustrate how to systematically construct a Lyapunov function using an LP in the planning scenario, and using function approximation in RL for guaranteeing safety during learning.

\section{Safe Reinforcement Learning Using Lyapunov Functions}
\vspace{-0.1in}

Motivated by the challenge of computing a Lyapunov function $L_{\epsilon^*}$ such that its induced set of policies contains $\pi^*$, in this section we approximate $ \epsilon^*$ with an auxiliary constraint cost $\widetilde\epsilon$, which is the \emph{largest} auxiliary cost that satisfies the Lyapunov condition: $L_{\widetilde \epsilon}(x)\geq T_{\pi_B,d}[L_{\widetilde \epsilon}](x),\,\,\forall x\in\mathcal X'$ and the safety condition $L_{\widetilde \epsilon}(x_0)\leq d_0$. 
The larger the $\widetilde\epsilon$, the larger the set of policies $\mathcal F_{L_{\widetilde \epsilon}}$.  Thus by choosing the largest such auxiliary cost, we hope to have a better chance of including the optimal policy $\pi^*$ in the set of feasible policies.
So, we consider the following LP problem:

\vspace{-0.125in}
\begin{small}
\begin{equation}\label{eq:opt_eps_baseline}
\widetilde \epsilon\in\!\arg\!\max_{\epsilon:\mathcal X'\rightarrow\mathbb R_{\geq 0}}\!\!\Bigg\{\!\sum_{x\in\mathcal X'}\epsilon(x)\!:d_0 - \mathcal D_{\pi_B}(x_0)\geq\mathbf 1(x_0)^\top(I- \{P(x'|x,\pi_B)\}_{x,x'\in\mathcal X'})^{-1}\epsilon
\!\Bigg\}.
\end{equation}
\end{small}
\vspace{-0.125in}

Here $\mathbf 1(x_0)$ represents a one-hot vector in which the non-zero element is located at $x=x_0$. 

On the other hand, whenever $\pi_B$ is a feasible policy, then the problem in \eqref{eq:opt_eps_baseline} always has a non-empty solution\footnote{This is due to the fact that $d_0 - \mathcal D_{\pi_B}(x_0)\geq 0$, and therefore $\widetilde \epsilon(x)=0$ is a feasible solution.}. 
Furthermore, notice that $\mathbf 1(x_0)^\top(I- \{P(x'|x,\pi_B)\}_{x,x'\in\mathcal X'})^{-1}\mathbf 1(x)$ represents the total visiting probability $\mathbb E[\sum_{t=0}^{\mathrm{T}^*-1}\mathbf 1\{x_t=x\}\mid x_0,\pi_B]$ from initial state $x_0$ to any state $x\in\mathcal X'$, which is a non-negative quantity. Therefore, using the extreme point argument in LP \citep{luenberger1984linear}, one can simply conclude that   
the maximizer of problem \eqref{eq:opt_eps_baseline} is an indicator function whose non-zero element locates at state $\underline x$ that corresponds to the minimum total visiting probability from $x_0$, i.e., $\widetilde \epsilon(x)={(d_0 - \mathcal D_{\pi_B}(x_0))\cdot\mathbf 1\{x=\underline x\}}/{\mathbb E[\sum_{t=0}^{\mathrm{T}^*-1}\mathbf 1\{x_t=\underline x\}\mid x_0,\pi_B]}\geq 0,\,\forall x\in\mathcal X'$, 
where $\underline x\in\arg\min_{x\in\mathcal X'}\mathbb E[\sum_{t=0}^{\mathrm{T}^*-1}\mathbf 1\{x_t=x\}\mid x_0,\pi_B]$. 
On the other hand, suppose we further restrict the structure of $\widetilde\epsilon(x)$ to be a constant function, i.e., $\widetilde\epsilon(x)=\widetilde\epsilon$, $\forall x\in\mathcal X'$. Then one can show that the maximizer is given by $\widetilde \epsilon(x)={(d_0 - \mathcal D_{\pi_B}(x_0))}/{\mathbb E[\mathrm{T}^*\mid x_0,\pi_B]}$, $\forall x\in\mathcal X'$, where $
\mathbf 1(x_0)^\top(I- \{P(x'|x,\pi_B)\}_{x,x'\in\mathcal X'})^{-1}[1,\ldots,1]^\top= \mathbb E[\mathrm{T}^*\mid x_0,\pi_B]$ is the expected stopping time of the transient MDP. In cases when computing the expected stopping time is expensive, then one reasonable approximation is to replace the denominator of $\widetilde\epsilon$ with the upper-bound $\overline{\mathrm{T}}$.

Using this Lyapunov function $L_{\widetilde \epsilon}$, we propose the safe policy iteration (SPI) in Algorithm \ref{alg:safe_PI}, in which the Lyapunov function is updated via \emph{bootstrapping}, i.e., at each iteration $L_{\widetilde \epsilon}$ is re-computed using \eqref{eq:opt_eps_baseline}, w.r.t. the current baseline policy. Properties of SPI are summarized in the following proposition. 

\vspace{-0.05in}
\begin{algorithm}
\begin{small}
\begin{algorithmic}
\STATE {\bf Input:} Initial feasible policy $\pi_0$;
\FOR{$k= 0,1,2,\ldots$}
\STATE {\bf Step 0:} With $\pi_b=\pi_k$, evaluate the Lyapunov function $L_{\epsilon_k}$, where $\epsilon_k$ is a solution of  \eqref{eq:opt_eps_baseline}
\STATE {\bf Step 1:} Evaluate the cost value function $V_{\pi_k}(x)=\mathcal C_{\pi_k}(x)$; Then update the policy by solving the following problem:
$
\pi_{k+1}(\cdot|x)\in\arg\!\min_{\pi\in\mathcal F_{L_{\epsilon_k}}(x)}T_{\pi,c}[V_{\pi_k}](x),\forall x\in\mathcal X'
$
\ENDFOR  
\STATE {\bf Return} Final policy $\pi_{k^*}$
\end{algorithmic}
\end{small}
\caption{Safe Policy Iteration (SPI)}
\label{alg:safe_PI}
\end{algorithm}
\vspace{-0.05in}

\begin{proposition}\label{prop:properties_safe_PI}
Algorithm \ref{alg:safe_PI} has following properties: \emph{(i) Consistent Feasibility}, i.e., suppose the current policy $\pi_k$ is feasible, then the updated policy $\pi_{k+1}$ is also feasible, i.e., $\mathcal D_{\pi_k}(x_0)\leq d_0$ implies $\mathcal D_{\pi_{k+1}}(x_0)\leq d_0$; (ii) Monotonic Policy Improvement, i.e., the cumulative cost induced by $\pi_{k+1}$ is lower than or equal to that by $\pi_k$, i.e., $\mathcal C_{\pi_{k+1}}(x)\leq \mathcal C_{\pi_k}(x)$ for any $x\in\mathcal X'$; {(iii) Convergence}, i.e., suppose a strictly concave regularizer is added to optimization problem \eqref{eq:opt_eps_baseline} and a strictly convex regularizer is added to policy optimization step. Then the policy sequence asymptotically converges.
\end{proposition}
The proof of this proposition is given in Appendix \ref{appendix:properties_safe_PI}, and the sub-optimality performance bound of SPI can be found in Appendix \ref{appendix:SPI_performance}.  Analogous to SPI, we also propose a safe value iteration (SVI), in which the Lyapunov function estimate is updated at every iteration via bootstrapping, using the current optimal value estimate. Details of SVI is given in Algorithm \ref{alg:safe_VI}, and its properties are summarized in the following proposition (whose proof is given in Appendix \ref{appendix:properties_safe_VI}). 

\begin{algorithm}[!ht]
\begin{small}
\begin{algorithmic}
\STATE {\bf Input:} Initial $Q-$function $Q_0$; Initial Lyapunov function $L_{\epsilon_0}$ w.r.t. auxiliary cost function $\epsilon_0(x)\!=\!0$;
\FOR{$k= 0,1,2,\ldots$}
\STATE{\bf Step 0:}  Compute $Q-$function
$
Q_{k+1}(x,a)=c(x,a)+\sum_{x'}\!P(x'|x,a)\min_{\pi\in\mathcal F_{L_{\epsilon_k}}(x')}\pi(\cdot|x')^\top Q_k(x',\cdot)
$
and policy $\pi_k(\cdot|x)\in\arg\min_{\pi\in\mathcal F_{L_{\epsilon_k}}(x)}\pi(\cdot|x)^\top Q_k(x,\cdot)$ 
\STATE {\bf Step 1:} With $\pi_B=\pi_k$, construct the Lyapunov function $L_{\epsilon_{k+1}}$, where $\epsilon_{k+1}$ is a solution of  \eqref{eq:opt_eps_baseline}; 
\ENDFOR  
\STATE {\bf Return} Final policy $\pi_{k^*}$
\end{algorithmic}
\end{small}
\caption{Safe Value Iteration (SVI)}
\label{alg:safe_VI}
\end{algorithm}

\begin{proposition}\label{prop:properties_safe_VI}
Algorithm \ref{alg:safe_VI} has following properties: {(i) Consistent Feasibility}; {(ii) Convergence}.
\end{proposition}
To justify the notion of bootstrapping, in both SVI and SPI, the Lyapunov function is updated based on the \emph{best} baseline policy (the policy that is feasible and by far has the lowest cumulative cost). Once the current baseline policy $\pi_{k}$ is \emph{sufficiently close} to an optimal policy $\pi^*$, then by Theorem \ref{thm:L_star} one concludes that the $L_{\widetilde \epsilon}-$induced set of policies contains an optimal policy. Although these algorithms do not have optimality guarantees, empirically they often return a near-optimal policy. 

In each iteration, the policy optimization step in SPI and SVI requires solving $|\mathcal X'|$ LP sub-problems, where each of them has $|\mathcal A|+2$ constraints and has a $|\mathcal A|-$dimensional decision-variable. Collectively, at each iteration its complexity is $O(|\mathcal X'||\mathcal A|^2(|\mathcal A|+2))$. While in the worst case SVI converges in $K=O(\overline{\mathrm{T}})$ steps \cite{bertsekas1995dynamic}, and SPI converges in $K=O(|\mathcal X'||\mathcal A|\overline{\mathrm{T}}\log \overline{\mathrm{T}})$ steps \citep{scherrer2013performance}, in practice $K$ is much smaller than $|\mathcal X'||\mathcal A|$. 
Therefore, even with the additional complexity of policy evaluation in SPI that is $O(\overline{\mathrm{T}} |\mathcal X'|^2)$, or the complexity of updating $Q-$function in SVI that is $O(|\mathcal A|^2 |\mathcal X'|^2)$, the complexity of these methods is $O(K|\mathcal X'||\mathcal A|^3+K|\mathcal X'|^2|\mathcal A|^2)$, which in practice is much lower than that of the dual LP method, whose complexity is $O(|\mathcal X'|^3|\mathcal A|^3)$ (see Section \ref{sec:existing approaches} for details).

\vspace{-0.1in}
\subsection{Lyapunov-based Safe RL Algorithms}\label{sec:safe_RL}
\vspace{-0.1in}

In order to improve scalability of SVI and SPI, we develop two \emph{off-policy} safe RL algorithms, namely safe DQN and safe DPI, which replace
the value and policy updates in safe DP with function approximations. Their pseudo-codes can be found in Appendix \ref{appendix:pseudo_code_safe_RL}. Before going into their details, we first introduce the policy distillation method, which will be later used in the safe RL algorithms. 

\paragraph{Policy Distillation:}
Consider the following LP problem for policy optimization in SVI and SPI:

\vspace{-0.125in}
\begin{small}
\begin{equation}\label{eq:approx_greedy_pol}
\pi'(\cdot|x)\in\arg\min_{\pi\in\Delta}\left\{\pi(\cdot|x)^\top  Q(x,\cdot): (\pi(\cdot|x)-\pi_B(\cdot|x))^\top Q_L(x, \cdot)\leq \widetilde\epsilon'(x)\right\},
\end{equation}
\end{small}
\vspace{-0.125in}

where $Q_L(x, a)=d(x)+\widetilde\epsilon'(x)+\sum_{x'}P(x'|x,a)L_{\widetilde\epsilon'}(x')$ is the state-action Lyapunov function. When the state-space is large (or continuous), explicitly solving for a policy becomes impossible without function approximation. 
Consider a parameterized policy $\pi_\phi$ with weights $\phi$. Utilizing the distillation concept \citep{rusu2015policy}, after computing the optimal action probabilities w.r.t. a batch of states, the policy $\pi_\phi$ is updated by solving
$\phi^*\in\arg\min_{\phi}\frac{1}{m}\sum_{m=1}^M\sum_{t=0}^{\overline{\mathrm{T}}-1} D_{\text{JSD}}(\pi_\phi(\cdot|x_{t,m})\parallel\pi'(\cdot|x_{t,m}))$, where the Jensen-Shannon divergence.
The pseudo-code of distillation is given in Algorithm \ref{alg:policy_distill}.

\paragraph{Safe $Q-$learning (SDQN):}
Here we sample an off-policy mini-batch of state-action-costs-next-state samples from the replay buffer and use it to update the value function estimates that minimize the MSE losses of Bellman residuals.
 Specifically, we first construct the state-action Lyapunov function estimate $\widehat Q_L(x, a;\theta_D,\theta_T)=\widehat Q_{D}(x,a;\theta_D)+\widetilde\epsilon'\cdot \widehat Q_{T}(x,a;\theta_T)$, by  learning the constraint value network $\widehat Q_{D}$ and stopping time value network $\widehat Q_{T}$ respectively. With a current baseline policy $\pi_k$, 
one can use function approximation to approximate the  auxiliary constraint cost (which is the solution of \eqref{eq:opt_eps_baseline},) by
$\widetilde\epsilon'(x)=\widetilde\epsilon'={(d_0 -\pi_k(\cdot|x_0)^\top \widehat Q_{D}(x_0,\cdot;\theta_D))}/{\pi_k(\cdot|x_{0})^\top \widehat Q_{T}(x_0,\cdot;\theta_T)}$.
Equipped with the Lyapunov function, in each iteration one can do a standard DQN update, except that the optimal action probabilities are computed via solving \eqref{eq:approx_greedy_pol}. Details of SDQN is given in Algorithm \ref{alg:safe_dqn}. 

\paragraph{Safe Policy Improvement (SDPI):}
Similar to SDQN, in this algorithm we first sample an off-policy mini-batch of samples from the replay buffer and use it to update the value function estimates (w.r.t. objective, constraint, and stopping-time estimate) that minimize MSE losses. Different from SDQN, in SDPI the value estimation is done using policy evaluation, which means that the objective $Q-$function is trained to minimize the Bellman residual w.r.t. actions generated by the current policy $\pi_k$, instead of the greedy actions. Using the same construction as in SDQN for auxiliary cost $\widetilde\epsilon'$, and state-action Lyapunov function $\widehat Q_L$, we then perform a policy improvement step by computing a set of greedy action probabilities from \eqref{eq:approx_greedy_pol}, and constructing an updated policy $\pi_{k+1}$ using policy distillation. 
Assuming the function approximations (for both value and policy) have low errors,  SDPI resembles several interesting properties from SPI, such as maintaining safety during training and improving policy monotonically. 
To improve learning stability, instead of the full policy update one can further consider a partial update $\pi_{k+1}=(1-\alpha) \pi_k + \alpha\pi'$, where $\alpha\in(0,1)$ is a \emph{mixing constant} that controls safety and exploration \citep{achiam2017constrained, kakade2002approximately}. 
Details of SDPI is summarized in Algorithm \ref{alg:safe_cpi}.

In terms of practical implementations, in Appendix \ref{appendix:practical} we include techniques to improve stability during training, to handle continuous action space, and to scale up policy optimization step in \eqref{eq:approx_greedy_pol}.

\vspace{-0.15in}
\section{Experiments}\label{sec:experiments}
\vspace{-0.1in}

\begin{figure}
\begin{center}
  \begin{tabular}{ccc}
    \includegraphics[width=0.24\columnwidth, angle =270]{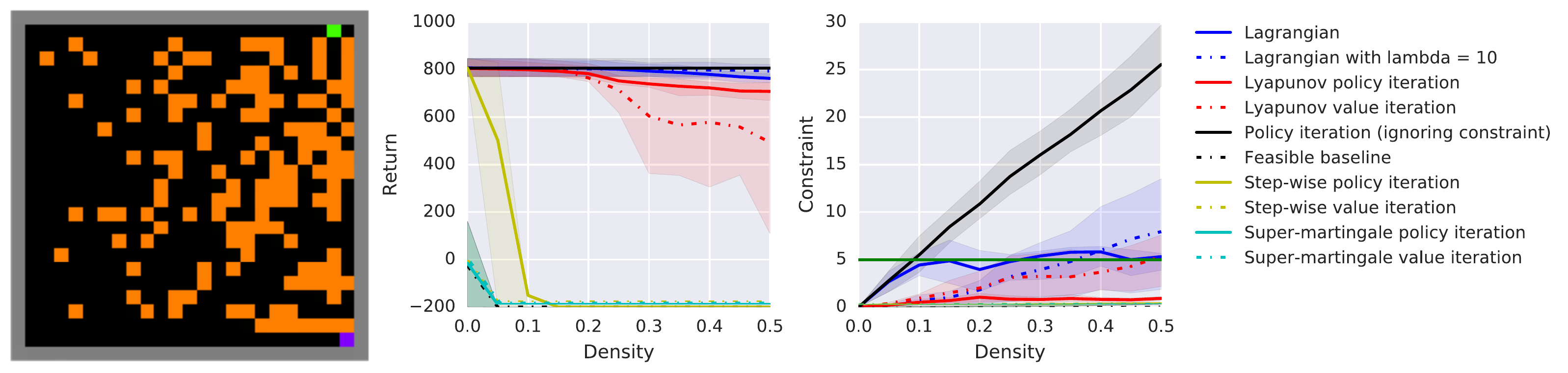} &
  \end{tabular}
\end{center}
  \vspace{-0.1in}
\caption{
Results of various planning algorithms on the grid-world environment with obstacles, with x-axis showing the obstacle density. From the leftmost column, the first figure illustrates the 2D planning domain example ($\rho=0.25$). The second and the third figure show the average return and the average cumulative constraint cost of the CMDP methods respectively. The fourth figure displays all the methods used in the experiment. The shaded regions indicate the $80\%$ confidence intervals. Clearly the safe DP algorithms compute policies that are safe and have good performance.}
\label{fig:results_planning}
  \vspace{-0.2in}
\end{figure}

Motivated by the safety issues of RL in \cite{leike2017ai}, we validate our safe RL algorithms using a stochastic 2D grid-world motion planning problem. In this domain, an agent (e.g., a robotic vehicle) starts in a safe region and its objective is to travel  to a given  destination.
At each time step the agent can move to any of its four neighboring states. Due to sensing and control noise, however, with probability $\delta$ a move to a random neighboring state occurs. To account for fuel usage, the stage-wise cost of each move until reaching the destination is $1$, while the reward achieved for reaching the destination is $1000$. Thus, we would like the agent to reach the destination in the shortest possible number of moves.  In between the starting point and the destination there is a number of obstacles that the agent may pass through but should avoid for safety; each time the agent is on an obstacle it incurs a constraint cost of $1$.  Thus, in the CMDP setting, the agent's goal is to reach the destination in the shortest possible number of moves while passing through obstacles at most $d_0$ times or less. 
For demonstration purposes, we choose a $25 \times 25$ grid-world (see Figure \ref{fig:results_planning}) with a total of $625$ states. We also have a  density ratio $\rho\in(0,1)$ that sets the obstacle-to-terrain ratio. When $\rho$ is close to $0$, the problem is obstacle-free, and if $\rho$ is close to $1$, then the problem becomes more challenging. In the normal problem setting, we choose a density $\rho=0.3$, an error probability $\delta=0.05$, a constraint threshold $d_0=5$, and a maximum horizon of $200$ steps. The initial state is located in $(24,24)$, and the goal is placed in $(0,\alpha)$, where $\alpha\in[0,24]$ is a uniform random variable. To account for statistical significance, the results of each experiment are averaged over $20$ trials.

\textbf{CMDP Planning:}
In this task we have explicit knowledge on reward function and transition probability. The main goal is to compare our safe DP algorithms (SPI and SVI) with the following common CMDP baseline methods: (i) \emph{Step-wise Surrogate}, (ii) \emph{Super-martingale Surrogate}, (iii) \emph{Lagrangian}, and (iv) \emph{Dual LP}. 
Since the methods in (i) and (ii) are surrogate algorithms, we will also evaluate these methods with both value iteration and policy iteration. 
To illustrate the level of sub-optimality, we will also compare the returns and constraint costs of these methods with baselines that are generated by maximizing return or minimizing constraint cost of two separate MDPs. 
The main objective here is to illustrate that safe DP algorithms are less conservative than other surrogate methods, are more numerically stable than the Lagrangian method, and are more computationally efficient than the Dual LP method (see Appendix \ref{appendix:experiment_setup}), without using function approximations. 

Figure \ref{fig:results_planning} presents the results on returns and the cumulative constraint costs of the aforementioned CMDP methods over a spectrum of $\rho$ values, ranging from $0$ to $0.5$. In each method, the initial policy is a conservative baseline policy $\pi_B$ that minimizes the constraint cost.
Clearly from the empirical results, although the polices generated by the four surrogate algorithms are feasible, they do not have significant policy improvements, i.e., return values are close to that of the initial baseline policy. Over all density settings, the SPI algorithm consistently computes a solution that is feasible and has good performance. The solution policy returned by SVI is always feasible, and it has near-optimal performance when the obstacle density is low. However, due to numerical instability its performance degrades as $\rho$ grows. Similarly, the Lagrangian methods return a near-optimal solution over most settings, but due to numerical issues their solutions start to violate constraint as $\rho$ grows. 

\textbf{Safe Reinforcement Learning:}
In this section we present the results of RL algorithms on this safety task.
We evaluate their learning performance on two variants: one in which the observation is a one-hot encoding the of the agent's location, and the other in which the observation is the 2D image representation of the grid map.  In each of these, we evaluate performance when $d_0=1$ and $d_0=5$.
We compare our proposed safe RL algorithms, SDPI and SDQN, to their unconstrained counterparts, DPI and DQN, as well as the Lagrangian approach to safe RL, in which the Lagrange multiplier is optimized via extensive grid search.  Details of the experimental setup is given in Appendix \ref{appendix:experiment_setup}. To make the tasks more challenging, we initialize the RL algorithms with a randomized baseline policy. 

Figure~\ref{fig:results2} shows the results of these methods across all task variants.  
Clearly, we see that SDPI and SDQN can adequately solve the tasks and compute agents with good return performance (similar to that of DQN and DPI in some cases), while guaranteeing safety. Another interesting observation in the SDQN and SDPI algorithms is that, once the algorithm finds a safe policy, then all updated policies \emph{remain} safe during throughout training. On the contrary, the Lagrangian approaches often achieve worse rewards and are more apt to violate the constraints during training \footnote{In Appendix \ref{appendix:experiment_setup}, we also report the results from the Lagrangian method, in which the Lagrange multiplier is learned using gradient ascent method \citep{chow2015risk}, and we observe similar (or even worse) behaviors.}, and the performance is very sensitive to initial conditions. Furthermore, in some cases (in experiment with $d_0=5$ and with discrete observation) the Lagrangian method cannot guarantee safety throughout training.

\begin{figure}
\begin{center}
  \begin{tabular}{ccccc}
    & \small Discrete obs, $d_0=5$ & \small Discrete obs, $d_0=1$ & \small Image obs, $d_0=5$& \small Image obs, $d_0=1$ \\
    \multirow{2}{*}{\rotatebox[origin=c]{90}{\tiny Constraints \hspace{2.1cm} Rewards\hspace{-1.7cm}}} &
    \includegraphics[width=0.2\columnwidth]{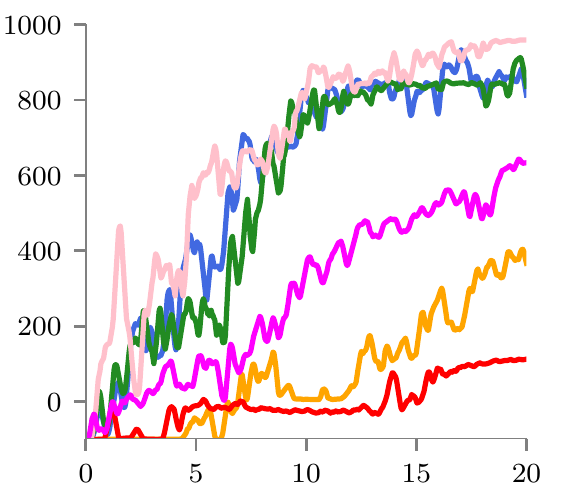} &
    \includegraphics[width=0.2\columnwidth]{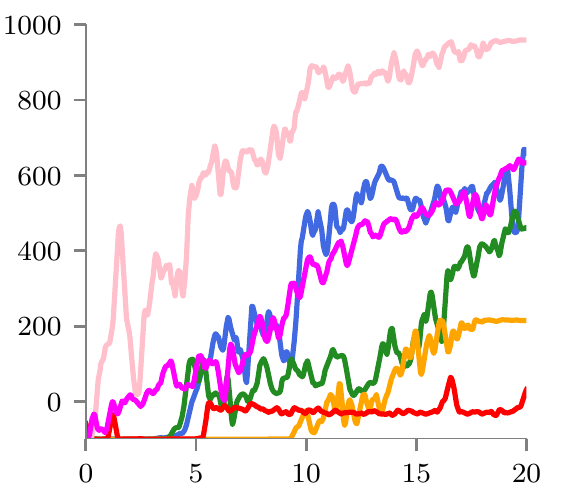} &
    \includegraphics[width=0.2\columnwidth]{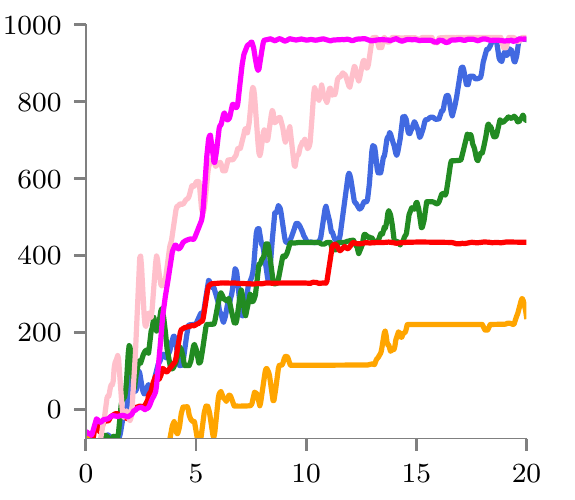} &
    \includegraphics[width=0.2\columnwidth]{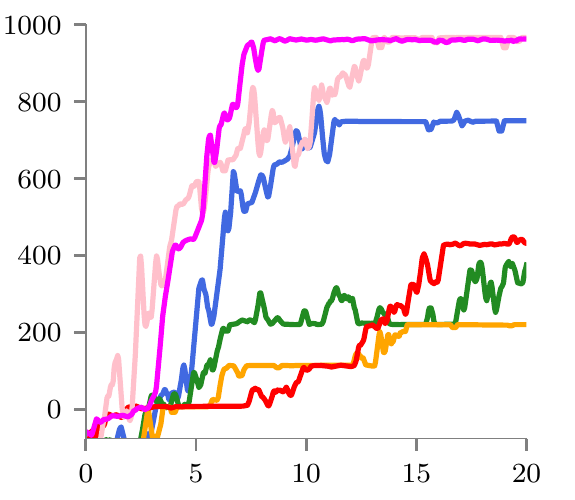} \\
    &
    \includegraphics[width=0.2\columnwidth]{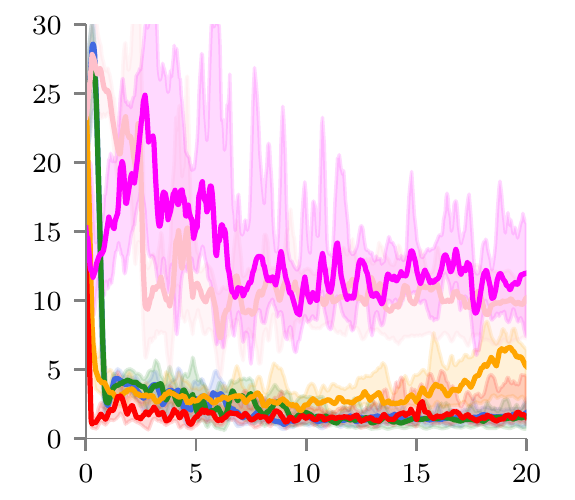} &
    \includegraphics[width=0.2\columnwidth]{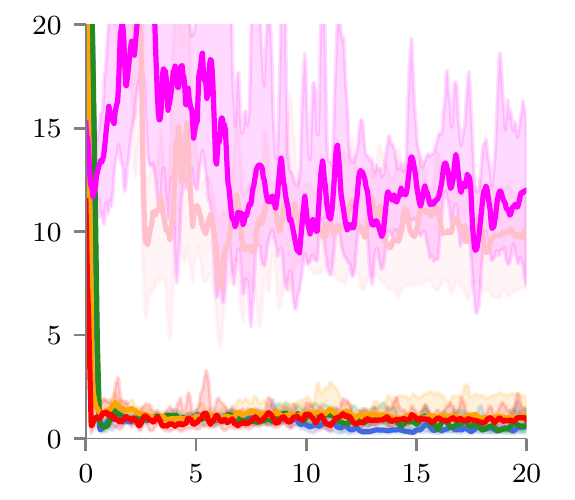} &
    \includegraphics[width=0.2\columnwidth]{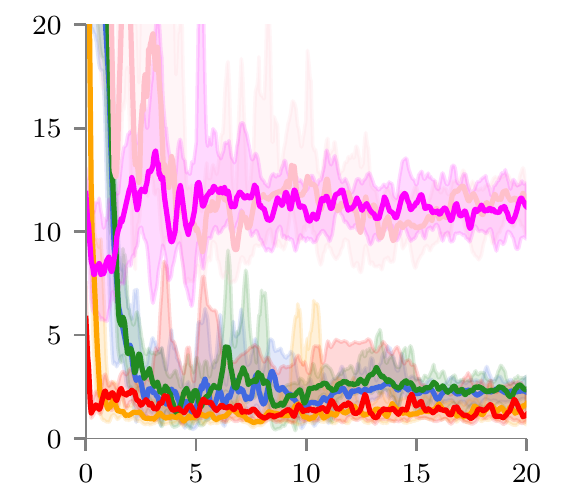} &
    \includegraphics[width=0.2\columnwidth]{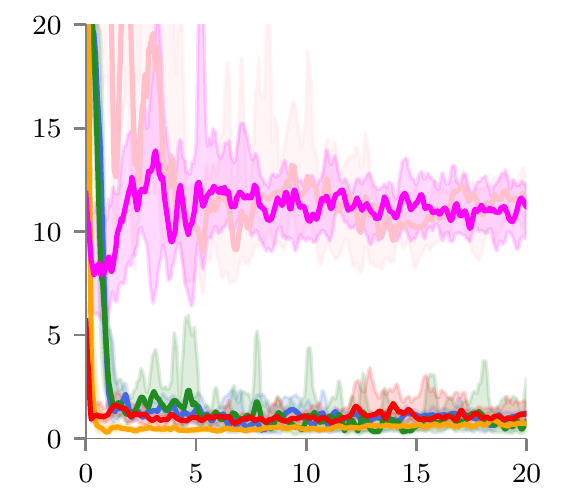} \\
    \multicolumn{5}{c}{\includegraphics[width=0.5\columnwidth]{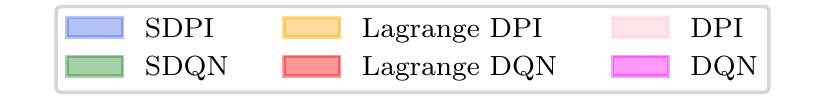}}
  \end{tabular}
\end{center}
  \vspace{-0.1in}
\caption{
Results of various RL algorithms on the grid-world environment with obstacles, with x-axis in thousands of episodes.  We include runs using discrete observations (a one-hot encoding of the agent's position) and image observations (showing the entire RGB 2D map of the world).  We discover that the Lyapunov-based approaches can perform safe learning, despite the fact that the environment dynamics model is not known and that deep function approximations are necessary.
}
  \vspace{-0.2in}
\label{fig:results2}
\end{figure}

\vspace{-0.15in}
\section{Conclusion}\label{sec:conclusion}
\vspace{-0.1in}

In this paper we formulated the problem of safe RL as a CMDP and proposed a \emph{novel} Lyapunov approach to solve CMDPs. We also derived an effective LP-based method to generate Lyapunov functions, such that the corresponding algorithm guarantees feasibility, and optimality under certain conditions. Leveraging these theoretical underpinnings, we showed how Lyapunov approaches can be used to transform DP (and RL) algorithms into their safe counterparts, that only requires straightforward modifications in the algorithm implementations. 
Empirically we validated our theoretical findings in using Lyapunov approach to guarantee safety and robust learning in RL. In general, our work represents a step forward in deploying RL to real-world problems in which guaranteeing safety is of paramount importance.  
Future research will focus on two directions. On
the algorithmic perspective, one major extension is to apply Lyapunov approach to policy gradient algorithms, and compare its performance with CPO in continuous RL problems. On the practical perspective, future work includes evaluating the Lyapunov-based RL algorithms on several real-world testbeds.  

\newpage
\bibliography{safety}

\begin{thebibliography}{10}

\bibitem{abe2010optimizing}
N.~Abe, P.~Melville, C.~Pendus, C.~Reddy, D.~Jensen, V.~Thomas, J.~Bennett,
  G.~Anderson, B.~Cooley, M.~Kowalczyk, et~al.
\newblock Optimizing debt collections using constrained reinforcement learning.
\newblock In {\em Proceedings of the 16th ACM SIGKDD international conference
  on Knowledge discovery and data mining}, pages 75--84, 2010.

\bibitem{achiam2017constrained}
J.~Achiam, D.~Held, A.~Tamar, and P.~Abbeel.
\newblock Constrained policy optimization.
\newblock In {\em International Conference of Machine Learning}, 2017.

\bibitem{altman1999constrained}
E.~Altman.
\newblock {\em Constrained {M}arkov decision processes}, volume~7.
\newblock CRC Press, 1999.

\bibitem{altman1998constrained}
Eitan Altman.
\newblock Constrained {M}arkov decision processes with total cost criteria:
  {L}agrangian approach and dual linear program.
\newblock {\em Mathematical methods of operations research}, 48(3):387--417,
  1998.

\bibitem{amodei2016concrete}
D.~Amodei, C.~Olah, J.~Steinhardt, P.~Christiano, J.~Schulman, and D.~Man{\'e}.
\newblock Concrete problems in {AI} safety.
\newblock {\em arXiv preprint arXiv:1606.06565}, 2016.

\bibitem{berkenkamp2017safe}
F.~Berkenkamp, M.~Turchetta, A.~Schoellig, and A.~Krause.
\newblock Safe model-based reinforcement learning with stability guarantees.
\newblock In {\em Advances in Neural Information Processing Systems}, pages
  908--919, 2017.

\bibitem{bertsekas1995dynamic}
D.~Bertsekas.
\newblock {\em Dynamic programming and optimal control}.
\newblock Athena scientific Belmont, MA, 1995.

\bibitem{boyd2004convex}
S.~Boyd and L.~Vandenberghe.
\newblock {\em Convex optimization}.
\newblock Cambridge university press, 2004.

\bibitem{el2016convex}
M~El Chamie, Y.~Yu, and B.~A{\c{c}}{\i}kme{\c{s}}e.
\newblock Convex synthesis of randomized policies for controlled {M}arkov
  chains with density safety upper bound constraints.
\newblock In {\em American Control Conference}, pages 6290--6295, 2016.

\bibitem{chow2015trading}
Y.~Chow, M.~Pavone, B.~Sadler, and S.~Carpin.
\newblock Trading safety versus performance: Rapid deployment of robotic swarms
  with robust performance constraints.
\newblock {\em Journal of Dynamic Systems, Measurement, and Control}, 137(3),
  2015.

\bibitem{chow2015risk}
Yinlam Chow, Mohammad Ghavamzadeh, Lucas Janson, and Marco Pavone.
\newblock Risk-constrained reinforcement learning with percentile risk
  criteria.
\newblock {\em arXiv preprint arXiv:1512.01629}, 2015.

\bibitem{dalal2018safe}
G.~Dalal, K.~Dvijotham, M.~Vecerik, T.~Hester, C.~Paduraru, and Y.~Tassa.
\newblock Safe exploration in continuous action spaces.
\newblock {\em arXiv preprint arXiv:1801.08757}, 2018.

\bibitem{gabor1998multi}
Z.~G{\'a}bor and Z.~Kalm{\'a}r.
\newblock Multi-criteria reinforcement learning.
\newblock In {\em International Conference of Machine Learning}, 1998.

\bibitem{geibel2005risk}
P.~Geibel and F.~Wysotzki.
\newblock Risk-sensitive reinforcement learning applied to control under
  constraints.
\newblock {\em Journal of Artificial Intelligence Research}, 24:81--108, 2005.

\bibitem{glynn2008bounding}
P.~Glynn, A.~Zeevi, et~al.
\newblock Bounding stationary expectations of markov processes.
\newblock In {\em Markov processes and related topics: a Festschrift for Thomas
  G. Kurtz}, pages 195--214. Institute of Mathematical Statistics, 2008.

\bibitem{gu2016continuous}
S.~Gu, T.~Lillicrap, I.~Sutskever, and S.~Levine.
\newblock Continuous deep {Q-}learning with model-based acceleration.
\newblock In {\em International Conference on Machine Learning}, pages
  2829--2838, 2016.

\bibitem{junges2016safety}
S.~Junges, N.~Jansen, C.~Dehnert, U.~Topcu, and J.~Katoen.
\newblock Safety-constrained reinforcement learning for {MDPs}.
\newblock In {\em International Conference on Tools and Algorithms for the
  Construction and Analysis of Systems}, pages 130--146, 2016.

\bibitem{kakade2002approximately}
S.~Kakade and J.~Langford.
\newblock Approximately optimal approximate reinforcement learning.
\newblock In {\em International Conference on International Conference on
  Machine Learning}, pages 267--274, 2002.

\bibitem{khalil1996noninear}
Hassan~K Khalil.
\newblock Noninear systems.
\newblock {\em Prentice-Hall, New Jersey}, 2(5):5--1, 1996.

\bibitem{lee2017first}
J.~Lee, I.~Panageas, G.~Piliouras, M.~Simchowitz, M.~Jordan, and B.~Recht.
\newblock First-order methods almost always avoid saddle points.
\newblock {\em arXiv preprint arXiv:1710.07406}, 2017.

\bibitem{leike2017ai}
J.~Leike, M.~Martic, V.~Krakovna, P.~Ortega, T.~Everitt, A.~Lefrancq,
  L.~Orseau, and S.~Legg.
\newblock Ai safety gridworlds.
\newblock {\em arXiv preprint arXiv:1711.09883}, 2017.

\bibitem{luenberger1984linear}
D.~Luenberger, Y.~Ye, et~al.
\newblock {\em Linear and nonlinear programming}, volume~2.
\newblock Springer, 1984.

\bibitem{mastronarde2011fast}
N.~Mastronarde and M.~van~der Schaar.
\newblock Fast reinforcement learning for energy-efficient wireless
  communication.
\newblock {\em IEEE Transactions on Signal Processing}, 59(12):6262--6266,
  2011.

\bibitem{mnih2013playing}
V.~Mnih, K.~Kavukcuoglu, D.~Silver, A.~Graves, I.~Antonoglou, D.~Wierstra, and
  M.~Riedmiller.
\newblock Playing {A}tari with deep reinforcement learning.
\newblock {\em arXiv preprint arXiv:1312.5602}, 2013.

\bibitem{moldovan2012safe}
T.~Moldovan and P.~Abbeel.
\newblock Safe exploration in {M}arkov decision processes.
\newblock {\em arXiv preprint arXiv:1205.4810}, 2012.

\bibitem{mossalam2016multi}
H.~Mossalam, Y.~Assael, D.~Roijers, and S.~Whiteson.
\newblock Multi-objective deep reinforcement learning.
\newblock {\em arXiv preprint arXiv:1610.02707}, 2016.

\bibitem{neely2010stochastic}
M.~Neely.
\newblock Stochastic network optimization with application to communication and
  queueing systems.
\newblock {\em Synthesis Lectures on Communication Networks}, 3(1):1--211,
  2010.

\bibitem{neu2017unified}
G.~Neu, A.~Jonsson, and V.~G{\'o}mez.
\newblock A unified view of entropy-regularized markov decision processes.
\newblock {\em arXiv preprint arXiv:1705.07798}, 2017.

\bibitem{ono2015chance}
M.~Ono, M.~Pavone, Y.~Kuwata, and J.~Balaram.
\newblock Chance-constrained dynamic programming with application to risk-aware
  robotic space exploration.
\newblock {\em Autonomous Robots}, 39(4):555--571, 2015.

\bibitem{perkins2002lyapunov}
T.~Perkins and A.~Barto.
\newblock Lyapunov design for safe reinforcement learning.
\newblock {\em Journal of Machine Learning Research}, 3:803--832, 2002.

\bibitem{regan2009regret}
K.~Regan and C.~Boutilier.
\newblock Regret-based reward elicitation for {M}arkov decision processes.
\newblock In {\em Proceedings of the Twenty-Fifth Conference on Uncertainty in
  Artificial Intelligence}, pages 444--451, 2009.

\bibitem{roijers2013survey}
D.~Roijers, P.~Vamplew, S.~Whiteson, and R.~Dazeley.
\newblock A survey of multi-objective sequential decision-making.
\newblock {\em Journal of Artificial Intelligence Research}, 48:67--113, 2013.

\bibitem{rusu2015policy}
A.~Rusu, S.~Colmenarejo, C.~Gulcehre, G.~Desjardins, J.~Kirkpatrick,
  R.~Pascanu, V.~Mnih, K.~Kavukcuoglu, and R.~Hadsell.
\newblock Policy distillation.
\newblock {\em arXiv preprint arXiv:1511.06295}, 2015.

\bibitem{schaul2015prioritized}
T.~Schaul, J.~Quan, I.~Antonoglou, and D.~Silver.
\newblock Prioritized experience replay.
\newblock {\em arXiv preprint arXiv:1511.05952}, 2015.

\bibitem{scherrer2013performance}
Bruno Scherrer.
\newblock Performance bounds for $\lambda$ policy iteration and application to
  the game of {T}etris.
\newblock {\em Journal of Machine Learning Research}, 14(Apr):1181--1227, 2013.

\bibitem{schmitt2006complexity}
M.~Schmitt and L.~Martignon.
\newblock On the complexity of learning lexicographic strategies.
\newblock {\em Journal of Machine Learning Research}, 7:55--83, 2006.

\bibitem{shani2005mdp}
G.~Shani, D.~Heckerman, and R.~Brafman.
\newblock An {MDP}-based recommender system.
\newblock {\em Journal of Machine Learning Research}, 6:1265--1295, 2005.

\bibitem{tamar2012policy}
A.~Tamar, D.~Di Castro, and S.~Mannor.
\newblock Policy gradients with variance related risk criteria.
\newblock In {\em International Conference of Machine Learning}, 2012.

\bibitem{hasselt2010double}
H.~van Hasselt.
\newblock Double {Q-}learning.
\newblock In {\em Advances in Neural Information Processing Systems}, pages
  2613--2621, 2010.

\end{thebibliography}

\newpage
\appendix
\section{Safety Constraints in Planning Problems}\label{sec:safety_problem_examples}
To motivate the CMDP formulation studied in this paper, in this section we include two real-life examples of modeling safety using the reachability constraint, and the constraint that limits the agent's visits to undesirable states.

\subsection{Reachability Constraint}
Reachability is a common concept in motion-planning and engineering applications, where for any given policy $\pi$ and initial state $x_0$,  the following the constraint function is considered:
\[
\mathbb P(\exists t\in\{0,1,\ldots, \mathrm{T}^*-1\}, x_t\in\mathcal S_{H}\mid x_0,\pi).
\]
Here $\mathcal S_{H}$ represents the real subset of hazardous regions for the states and actions. Therefore, the constraint cost represents the probability of reaching an unsafe region at any time before the state reaches the terminal state.
To further analyze this constraint function, one notices that
\[
\begin{split}
\mathbb P(\exists t\in\{0,1,\ldots, \mathrm{T}^*-1\}, x_t\in\mathcal S_{H}\mid x_0,\pi)=&\mathbb P\left(\bigcup_{t=0}^{\mathrm{T}^*-1}\bigcap_{j=0}^{t-1}\{x_j\not\in\mathcal S_{H}\}\cap \{x_t\in\mathcal S_{H}\}\mid x_0,\pi\right)\\
=&\mathbb E\left[\sum_{t=0}^{\mathrm{T}^*-1}\prod_{j=0}^{t-1}\mathbf 1\{x_j\not\in\mathcal S_{H}\}\cdot\mathbf 1\{x_t\in\mathcal S_{H}\}\mid x_0,\pi\right].
\end{split}
\]
In this case, a policy $\pi$ is deemed \emph{safe} if the reachability probability to the unsafe region is bounded by threshold $d_0\in(0,1)$, i.e.,
\begin{equation}\label{eq:reachability_constraint}
\mathbb P(\exists t\in\{0,1,\ldots,\mathrm{T}^*-1\}, x_t\in\mathcal S_{H}\mid x_0,\pi)\leq d_0.
\end{equation}

To transform the reachability constraint into a standard CMDP constraint, we define an additional state $s\in\{0,1\}$ that keeps track of the reachability status at time $t$. Here $s_t=1$ indicates the system has never visited a hazardous region up till time $t$, and otherwise $s_t=0$. Let $s_0=1$, we can easily see that by defining the following deterministic transition
\[
s_t=s_{t-1}\cdot\mathbf 1\{x_t\not\in\mathcal S_{H}\},\,\,\forall t\geq 1,
\]
$s_t$ has the following formulation:
$
s_t=\prod_{j=0}^{t-1}\mathbf 1\{x_j\not\in\mathcal S_{H}\}.
$

Collectively, with the state augmentation $\hat x=(x,s)$, one defines the augmented CMDP $(\hat{\mathcal X},\mathcal A,C, \hat D,\hat P, \hat x_0, d_0)$, where $\hat{\mathcal X}=\mathcal X\times\{0,1\}$ is the augmented state space, $\hat d(\hat x)=s \cdot d(x)$
is the augmented constraint cost,
$\hat P(\hat x'|\hat x,a)=P(x'|x,a)\cdot\mathbf 1\{s'=s\cdot\mathbf 1\{x\not\in\mathcal S_{H}\}\}$ is the augmented transition probability, and $\hat{x_0}=(x_0,1)$ is the initial (augmented) state. By using this augmented CMDP, immediately the reachability constraint is equivalent to
$
\mathbb E\big[\sum_{t=0}^{\mathrm{T}^*-1}\hat d(\hat x_t)\mid x_0,\pi\big]\leq d_0.
$

\subsection{Constraint w.r.t. Undesirable Regions of States}
Consider the notion of safety where one restricts the total visiting frequency of an agent to an undesirable region (of states). This notion of safety appears in applications such as system maintenance, in which the system can only tolerate its state to visit  (in expectation) a hazardous region, namely $\mathcal S_{H}$, for a fixed number of times. Specifically, for given initial state $x_0$, consider the following constraint that bounds the total frequency of visiting $\mathcal S_{H}$ with a pre-defined threshold $d_0$, i.e.,
$
\mathbb E\big[\sum_{t=0}^{\mathrm{T}^*-1} d(x_t)\mid x_0,\pi\big]\leq d_0,
$
where $d(x_t)=\mathbf 1\{x_t\in\mathcal S_{H}\}$.
To model this notion of safety using a CMDP, one can rewrite the above constraint using the constraint immediate cost $d(x)=\mathbf 1\{x\in\mathcal S_{H}\}$, and the constraint threshold $d_0$. 
To study the connection between the reachability constraint, and the above constraint w.r.t. undesirable region, notice that 
\[
\mathbb  E\left[\sum_{t=0}^{\mathrm{T}^*-1}\prod_{j=0}^{t-1}\mathbf 1\{x_j\not\in\mathcal S_{H}\}\cdot\mathbf 1\{x_t\in\mathcal S_{H}\}\mid x_0,\pi\right]\leq \mathbb E\left[\sum_{t=0}^{\mathrm{T}^*-1}\mathbf 1\{x_t\in\mathcal S_{H}\}\mid x_0,\pi\right].
\]
This clearly indicates that any policies which satisfies the constraint w.r.t. undesirable region, also satisfies the reachability constraint.

\newpage
\section{Existing Approaches for Solving CMDPs}\label{sec:existing approaches}
Before going to the main result, we first revisit several existing CMDP algorithms in the literature, which later serve as the baselines for comparing with our safe CMDP algorithms. For the sake of brevity, we will only provide an overview of these approaches here and defer their details to Appendix \ref{sec:existing_approaches_details}.

\paragraph{The Lagrangian Based Algorithm:}
The standard way of solving problem $\mathcal{OPT}$ is by applying the Lagrangian method. To start with, consider the following minimax problem:
$
\min_{\pi\in\Delta}\max_{\lambda\geq 0} \,\,\mathcal C_\pi(x_0) + \lambda (\mathcal D_\pi(x_0) - d_0),
$
where $\lambda$ is the Lagrange multiplier w.r.t. the CMDP constraint. According to Theorem 9.9 and Theorem 9.10 in \cite{altman1999constrained}, the optimal policy $\pi^*$ of problem $\mathcal{OPT}$ can be calculated by solving the following Lagrangian function
$
\pi^*\in\arg\min_{\pi\in\Delta} \mathcal C_\pi(x_0) + \lambda^*_{d_0} \mathcal D_\pi(x_0),
$
where $\lambda^*_{d_0}$ is the optimal Lagrange multiplier. 
Utilizing this result, one can compute the saddle point pair $(\pi^*,\lambda^*)$ using primal-dual iteration. Specifically, for a given $\lambda\geq 0$, solve the policy minimization problem using standard dynamic programming with $\lambda-$parametrized Bellman operator $T_\lambda[V](x)=\min_{\pi\in\Delta(x)}T_{\pi,c+\lambda d}[V](x)$ if $x\in\mathcal X$; For a given policy $\pi$, solve for the following linear optimization problem: $\max_{\lambda\geq 0} \,\,\mathcal C_\pi(x_0) + \lambda (\mathcal D_\pi(x_0) - d_0)$. Based on Theorem 9.10 in \cite{altman1999constrained}, this procedure will asymptotically converge to the saddle point solution.
However, this algorithm presents several major challenges. 
(i) In general there is no known convergence rate guarantees, several studies  \citep{lee2017first} also showed that using primal-dual first-order iterative method to find saddle point may run into numerical instability issues; (ii) Choosing a good initial estimate of the Lagrange multiplier is not intuitive; (iii) Following the same arguments from \cite{achiam2017constrained}, during iteration the policy may be infeasible w.r.t. problem $\mathcal{OPT}$, and feasibility is guaranteed after the algorithm converges. This is hazardous in RL when one needs to execute the intermediate policy (which may be unsafe) during training. 

\paragraph{The Dual LP Based Algorithm:}
Another method of solving problem $\mathcal{OPT}$ is based on computing its \emph{occupation measures} w.r.t. the optimal policy. In transient MDPs, for any given policy $\pi$ and initial state $x_0$, the state-action occupation measure is $
\rho_\pi(x,a)=\mathbb E\left[\sum_{t=0}^{\mathrm{T}^*-1}\mathbf 1\{x_t=x,a_t=a\}\mid x_0,\pi\right]
$, which characterizes the total visiting probability of state-action pair $(x,a)\in\mathcal X\times\mathcal A$, induced by policy $\pi$ and initial state $x_0$. Utilizing this quantity, Theorem 9.13 in \cite{altman1999constrained}, has shown that problem $\mathcal{OPT}$ can be reformulated as a linear programming (LP) problem (see Equation \eqref{eq:dual_LP_1} to \eqref{eq:dual_LP_2} in Appendix \ref{sec:existing_approaches_details}), whose decision variable is of dimension $|\mathcal X'||\mathcal A|$, and it has $2|\mathcal X'||\mathcal A| + 1$ constraints. Let $\rho^*$ be the solution of this LP, the optimal Markov stationary policy is given by $\pi^*(a|x)={\rho^*(x,a)}/{\sum_{a\in\mathcal A}\rho^*(x,a)}$. To solve this problem, one can apply the standard algorithm such as interior point method, which is a strong polynomial time algorithm with complexity $O(|\mathcal X'|^2|\mathcal A|^2 (2|\mathcal X'||\mathcal A| + 1))$ \citep{boyd2004convex}.
While this is a straight-forward methodology, it can only handle CMDPs with finite state and action spaces. Furthermore, this approach is computationally expensive when the size of these spaces are large. To the best of our knowledge, it is also unclear how to extend this approach to RL, when transition probability and immediate reward/constraint reward functions are unknown.

\paragraph{Step-wise Constraint Surrogate Approach:}
This approach transforms the multi-stage CMDP constraint into a sequence of step-wise constraints, where each step-wise constraint can be directly embedded into set of admissible actions in the Bellman operator. 
To start with, for any state $x\in\mathcal X'$, consider the following feasible set of policies:
$
\mathcal F^{\text{SW}}(d_0,x)=\big\{\pi\in\Delta:\sum_{a\in\mathcal A}\sum_{x'\in\mathcal X'}\pi(a|x)P(x'|x,a)d(x')\leq \frac{d_0}{\overline{\mathrm{T}}}\big\},
$
where $\overline{\mathrm{T}}$ is the upper-bound of the MDP stopping time. Based on \eqref{eq:statewise_constraint} in Appendix \ref{sec:existing_approaches_details}, one deduces that every policy $\pi$ in $\bigcap_{x\in\mathcal X'}\mathcal F^{\text{SW}}(d_0,x)$ is a feasible policy w.r.t. problem $\mathcal{OPT}$. Motivated by this observation, a solution policy can be solved by 
$\min_{\pi \in \bigcap_{x\in\mathcal X'}\mathcal F^{\text{SW}}(d_0,x)} \,\mathbb E\left[\sum_{t=0}^{\mathrm{T}^*-1}c(x_t, a_t)\mid x_0,\pi\right]$.
One benefit of studying this surrogate problem is that its solution satisfies the Bellman optimality condition w.r.t. the step-wise Bellman operator as $T^{\text{SW}}[V](x)=\min_{\pi\in\mathcal F^{\text{SW}}(d_0,x)}\sum_{a}\pi(a|x)\big[c(x,a)+\sum_{x'\in\mathcal X'}P(x'|x,a)V(x')\big]$ for any $x\in\mathcal X'$.
In particular $T^{\text{SW}}$ is a contraction operator, which implies that there exists a unique solution $V^{*,\text{SW}}$ to fixed point equation $T^{\text{SW}}[V](x)=V(x)$ for $x\in\mathcal X'$ such that  $V^{*,\text{SW}}(x_0)$ is a solution to the surrogate problem. Therefore this problem can be solved by standard DP methods such as value iteration or policy iteration. Furthermore, based on the structure of $\mathcal F^{\text{SW}}(d_0,x)$, any surrogate policy is feasible w.r.t. problem $\mathcal{OPT}$. 
However, the major drawback is that the step-wise constraint in $\mathcal F^{\text{SW}}(d_0,x)$ can be much more stringent than the original safety constraint in problem $\mathcal{OPT}$.

\paragraph{Super-martingale Constraint Surrogate Approach:}
This surrogate algorithm is originally proposed by \cite{gabor1998multi}, where the CMDP constraint is reformulated as the surrogate value function $\mathcal{DS}_{\pi}(x)=\max\left\{d_0,\mathcal D_\pi(x)\right\}$ at initial state $x\in\mathcal X'$. It has been shown that an arbitrary policy $\pi$ is a feasible policy of the CMDP if and only if $\mathcal{DS}_{\pi}(x_0)=d_0$. Notice that $\mathcal{DS}_{\pi}$ is known as a \emph{super-martingale surrogate}, due to the inequality $\mathcal{DS}_{\pi}(x)\leq T^{DS}_\pi[\mathcal{DS}_{\pi}](x)$ with respect to the contraction Bellman operator $T^{DS}_\pi[V](x)=\sum_{a\in\mathcal A}\pi(a|x)\max\left\{d_0,d(x)+\sum_{x'\in\mathcal X'}P(x'|x,a)V(x')\right\}$ of the constraint value function.
However, for arbitrary policy $\pi$, in general it is non-trivial to compute the value function $\mathcal{DS}_{\pi}(x)$, and instead one can easily compute its upper-bound value function $\overline{\mathcal{DS}}_{\pi}(x)$ which is the solution of the fixed-point equation
$V(x)= T^{DS}_\pi[V](x), \,\,\,\forall x\in\mathcal X'$, using standard dynamic programming techniques. 
To better understand how this surrogate value function guarantees feasibility in problem $\mathcal{OPT}$, at each state $x\in\mathcal X'$  consider the optimal value function of the minimization problem $\overline{\mathcal{DS}}(x)=\min_{\pi\in\Delta}  \overline{\mathcal{DS}}_{\pi}(x)$.
Then whenever $\overline{\mathcal{DS}}(x_0)\leq d_0$, the corresponding solution policy $\pi$ is a feasible policy of problem $\mathcal{OPT}$, i.e., $\mathcal D_\pi(x_0)\leq\overline{\mathcal{DS}}_{\pi}(x_0)=d_0$.
Now define $\mathcal F^{\text{DS}}(x):=\!\big\{\pi\in\Delta:T^{DS}_\pi[\overline{\mathcal{DS}}](x)=\! \overline{\mathcal{DS}}(x)\!\big\}$ as the set of refined feasible policies induced by $\overline{\mathcal{DS}}$. If the condition $\overline{\mathcal{DS}}(x_0)\leq d_0$ holds, then all the policies in $\mathcal F^{\text{DS}}(x)$ are feasible w.r.t. problem $\mathcal{OPT}$.
Utilizing this observation, a surrogate solution policy of problem $\mathcal{OPT}$ can be found by computing the solution policy of the fixed-point  equation $T_{\mathcal F^{\text{DS}}}[V](x)=V(x)$, for $x\in\mathcal X'$, where $T_{\mathcal F^{\text{DS}}}[V](x)=\min_{\pi\in\mathcal F^{\text{DS}}(d_0,x)}T_{\pi,c}[V](x)$. Notice that $T_{\mathcal F^{\text{DS}}}$ is a contraction operator, this procedure can also be solved using standard DP methods. 
The major benefit of this 2-step approach is that the computation of the feasibility set is decoupled from solving the optimization problem.  This allows us to apply approaches such as the \emph{lexicographical ordering} method from multi-objective stochastic optimal control methods \citep{roijers2013survey} to solve the CMDP, for which the constraint value function has a higher lexicographical order than the objective value function. 
However, since the refined set of feasible policies is constructed prior to policy optimization, it might still be overly conservative. Furthermore, even if there exists a non-trivial solution policy to the surrogate problem, characterizing its sub-optimality performance bound remains a challenging task.

\subsection{Details of Existing Solution Algorithms}\label{sec:existing_approaches_details}
In this section, we provide the details of the existing algorithms for solving CMDPs.

\paragraph{The Lagrangian Based Algorithm:}
The standard way of solving problem $\mathcal{OPT}$ is by applying the Lagrangian method. To start with, consider the following minimax problem:
\[
\min_{\pi\in\Delta}\max_{\lambda\geq 0} \,\,\mathcal C_\pi(x_0) + \lambda (\mathcal D_\pi(x_0) - d_0),
\]
where $\lambda$ is the Lagrange multiplier of the CMDP constraint, and the Lagrangian function is given by
\[
\mathcal L_{x_0,d_0}(\pi,\lambda)=\mathcal C_\pi(x_0) + \lambda (\mathcal D_\pi(x_0) - d_0)=\mathbb E\left[\sum_{t=0}^{\mathrm{T}^*-1}c(x_t,a_t)+\lambda d(x_t)\mid\pi,x_0\right]-\lambda d_0.
\]
A solution pair $(\pi^*,\lambda^*)$ is considered as a saddle point of Lagrangian function $\mathcal L_{x_0,d_0}(\pi,\lambda)$ if the following condition holds:
\[
\mathcal L_{x_0,d_0}(\pi^*,\lambda)\leq L_{x_0,d_0}(\pi,\lambda)\leq \mathcal L_{x_0,d_0}(\pi,\lambda^*),\,\,\forall \pi, \lambda\geq 0.
\]
According to Theorem 9.10 in \cite{altman1999constrained}, suppose the interior set of feasible set of problem $\mathcal{OPT}$ is non-empty, then there exists a solution pair $(\pi^*,\lambda^*)$ to the minimax problem that is a saddle point of Lagrangian function $\mathcal L_{x_0,d_0}(\pi,\lambda)$. 
Furthermore, Theorem 9.9 in \cite{altman1999constrained} shows that strong duality holds:
\[
\min_{\pi\in\Delta}\max_{\lambda\geq 0}\,\,\mathcal L_{x_0,d_0}(\pi,\lambda)=\max_{\lambda\geq 0}\min_{\pi\in\Delta}\,\,\mathcal L_{x_0,d_0}(\pi,\lambda).
\]
This implies that the optimal policy $\pi^*\in\Delta$ can be calculated by solving the following Lagrangian function
$
\pi^*\in\arg\min_{\pi\in\Delta} \mathcal L_{x_0,d_0}(\pi,\lambda^*),
$
with optimal Lagrange multiplier $\lambda^*$. 

Utilizing the structure of the Lagrangian function $\mathcal L_{x_0,d_0}(\pi,\lambda)$, for any fixed Lagrange multiplier $\lambda\geq 0$, consider the $\lambda-$Bellman operator $T_\lambda[V]$, where 
\[
T_\lambda[V](x)=\left\{\begin{array}{cl}
\min_{\pi\in\Delta(x)}\,T_{\pi,c+\lambda d}[V](x) & \text{if $x\in\mathcal X'$,}\\
0 & \text{otherwise}.
 \end{array}\right.
 \]
Since $T_\lambda[V]$ is a $\gamma-$contraction operator, by the Bellman principle of optimality, there is a unique solution $V^*$ to the fixed point equation $T_\lambda[V](x)=V(x)$ for $x\in\mathcal X$, which can be solved by dynamic programming algorithms, such as value iteration or policy iteration. Furthermore, the $\lambda-$optimal policy $\pi^*_\lambda$ has the form of
\[
\pi^*_\lambda(\cdot|x)\in\arg\min_{\pi\in\Delta(x)}T_{\pi,c+\lambda d}[V^*](x) \quad \text{if $x\in\mathcal X'$, }
\]
and $\pi^*_\lambda(\cdot|x)$ is an arbitrary probability distribution function if $x\not\in\mathcal X'$.

\paragraph{The Dual LP Based Algorithm:}
The other commonly-used method for solving problem $\mathcal{OPT}$ is based on computing its \emph{occupation measures} w.r.t. the optimal policy. In a transient MDP, for any given policy $\pi$ and initial state $x_0\in\mathcal X'$ the state-action occupation measure is defined as
\[
\rho_{\pi,x_0}(x,a)=\mathbb E\left[\sum_{t=0}^{\mathrm{T}^*-1}\mathbf 1\{x_t=x,a_t=a\}\mid x_0,\pi\right],\,\,\forall x\in\mathcal X, \,\,\forall a\in\mathcal A,
\]
and the occupation measure at state $x\in\X$ is defined as $\rho_{\pi,x_0}(x)=\sum_{a\in\mathcal A}\rho_{\pi,x_0}(x,a)$. Clearly these two occupation measures are related by the following property: $\rho_{\pi,x_0}(x,a)=\rho_{\pi,x_0}(x)\cdot \pi(a|x)$.
Furthermore, using the fact that a occupation measure $\rho_{\pi,x_0}(x,a)$ is indeed the sum of visiting distribution of the transient MDP induced by policy $\pi$, one clearly sees that it satisfies the following set of constraints:
\[
\begin{split}
\mathcal Q(x_0)\!=\!\bigg\{\!\rho:\mathcal X'\times\mathcal A\rightarrow\mathbb R:&\rho(x,a)\geq 0,\forall x\in\mathcal X',a\in\mathcal A,\\
&\!\!\!\sum_{x_p\in\mathcal X',a\in\mathcal A}\!\!\rho(x_p,a)(\mathbf 1\{x_p=x\}\!-\!P(x|x_p,a))\!=\!\mathbf 1\{x=x_0\}
\bigg\}.
\end{split}
\]
Therefore, by Theorem 9.13 in \cite{altman1999constrained}, equivalently problem $\mathcal{OPT}$ can be solved by the LP optimization problem with $2|\mathcal X'||\mathcal A| + 1$ constraints:
\begin{alignat}{2}
\min_{\rho\in\mathcal Q(x_0)} & & \quad&\mathcal \sum_{x\in\mathcal X',a\in\mathcal A}\rho(x,a) c(x,a) \label{eq:dual_LP_1}\\ 
\text{subject to} & & \quad&\sum_{x\in\mathcal X',a\in\mathcal A}\rho(x,a) d(x)\leq d_0,\label{eq:dual_LP_2}
\end{alignat}
and equipped with the minimizer minimizer $\rho^*\in \mathcal Q(x_0)$, the (non-uniform) optimal Markovian stationary policy is given by the following form:
\[
\pi^*(a|x)=\frac{\rho^*(x,a)}{\sum_{a\in\mathcal A}\rho^*(x,a)},\,\,\forall x\in\mathcal X', \,\,\forall a\in\mathcal A.
\]

\paragraph{Step-wise Constraint Surrogate Approach:}
To start with, without loss of generality assume that the agent is safe at the initial phase, i.e., 
$d(x_0)\leq 0$.
For any state $x\in\mathcal X'$, consider the following feasible set of policies:
\[
\mathcal F^{\text{SW}}(d_0,x)=\left\{\pi\in\Delta(x):\sum_{a\in\mathcal A}\sum_{x'\in\mathcal X'}\pi(a|x)P(x'|x,a)d(x')\leq \frac{d_0}{\overline{\mathrm{T}}}\right\},
\]
where $\overline{\mathrm{T}}$ is the uniform upper-bound of the random stopping time in the transient MDP. 
Immediately, for any policy $\pi\in\bigcap_{x\in\mathcal X'}\mathcal F^{\text{SW}}(d_0,x)$, one has the following inequality: 
\begin{equation}\label{eq:statewise_constraint}
\begin{split}
\mathbb E\left[\sum_{t=0}^{\mathrm{T}^*-1}d(x_t)\mid\pi,x_0\right]=&d(x_0)+\sum_{x\in\mathcal X',a\in\mathcal A}\rho_{\pi,x_0}(x,a)\sum_{x'\in\mathcal X'}P(x'|x,a)d(x')\\
\leq&\mathbb E\left[\sum_{t=0}^{\mathrm{T}^*-1}\frac{d_0}{\overline{\mathrm{T}}}\mid\pi,x_0\right]\leq d_0,
\end{split}
\end{equation}
where $\rho_{\pi,x_0}(x,a)=\mathbb E\left[\sum_{t=0}^{\mathrm{T}^*-1}\mathbf 1\{x_t=x\}|\pi,x_0\right]$ is the state-action occupation measure with initial state $x_0$ and policy $\pi$, which implies that the policy $\pi$ is safe, i.e., $\mathcal D_\pi(x_0)\leq d_0$. 
Equipped with this property, we propose the following surrogate problem for problem $\mathcal{OPT}$, whose solution (if exists) is guaranteed to be safe:
\begin{quote} {\bf Problem $\mathcal{OPT}^{\text{SW}}$:} Given an initial state $x_0$, a threshold $d_0$, solve
$
\min_{\pi \in \bigcap_{x\in\mathcal X'}\mathcal F^{\text{SW}}(d_0,x)} \,\mathbb E\left[\sum_{t=0}^{\mathrm{T}^*-1}c(x_t, a_t)\mid x_0,\pi\right].
$
\end{quote}
To solve problem $\mathcal{OPT}^{\text{SW}}$, for each state $x\in\mathcal X'$ define the step-wise Bellman operator as 
\[
T^{\text{SW}}[V](x)=\min_{\pi\in\mathcal F^{\text{SW}}(d_0,x)}T_{\pi,c}[V](x).
\]
Based on the standard arguments in \cite{bertsekas1995dynamic}, the Bellman operator $T^{\text{SW}}$ is a contraction operator, which implies that there exists a unique solution $V^{*,\text{SW}}$ to the fixed-point equation $T^{\text{SW}}[V](x)=V(x)$, for any $x\in\mathcal X'$, such that $V^{*,\text{SW}}(x_0)$ is a solution to problem $\mathcal{OPT}^{\text{SW}}$. 

\paragraph{Super-martingale Constraint Surrogate Approach:}
The following surrogate algorithm is proposed by \cite{gabor1998multi}. Before going to the main algorithm, first consider the following surrogate constraint value function w.r.t. policy $\pi\in\Delta$ and state $x\in\mathcal X'$:
\[
\begin{split}
\mathcal{DS}_{\pi}(x)=&\max\left\{d_0,\mathcal D_\pi(x)\right\}=\max\left\{\!d_0,d(x)+\!\!\sum_{a\in\mathcal A}\pi(a|x)\!\sum_{x'\in\mathcal X'}\!\!P(x'|x,a)\mathcal D_\pi(x')\!\right\}\\
\leq&\sum_{a\in\mathcal A}\!\pi(a|x)\!\max\left\{\!d_0,d(x)+\!\!\sum_{x'\in\mathcal X'}P(x'|x,a)\mathcal D_\pi(x')\!\right\}.
\end{split}
\]
The last inequality is due to the fact that the $\max$ operator is convex.
Clearly, by definition one has 
\[
\mathcal{DS}_{\pi}(x)\geq \mathcal D_\pi(x) \,\,\text{for each $x\in\mathcal X'$.}
\]
On the other hand, one also has the following property: $\mathcal{DS}_{\pi}(x_0)=d_0$ if and only if the constraint of problem $\mathcal{OPT}$ is satisfied, i.e., $\mathcal D_\pi(x_0)\leq d_0$. Now, by utilizing the contraction operator
\[
T^{DS}_\pi[V](x)=\sum_{a\in\mathcal A}\pi(a|x)\max\left\{d_0,d(x)+\sum_{x'\in\mathcal X'}P(x'|x,a)V(x')\right\},
\]
w.r.t. policy $\pi$, and by utilizing the definition of the constraint value function $\mathcal{DS}_{\pi}$, one immediately has the chain of inequalities:
\[
\mathcal{DS}_{\pi}(x)\leq T^{DS}_\pi[\mathcal D_\pi](x)
\leq T^{DS}_\pi[\mathcal{DS}_{\pi}](x),
\]
which implies that the value function $\{\mathcal{DS}_{\pi}(x)\}_{x\in\mathcal X'}$ is a \emph{super-martingale}, i.e.,
$\mathcal{DS}_{\pi}(x)\leq T^{DS}_\pi[\mathcal{DS}_{\pi}](x)$, for $x\in\mathcal X'$.

However in general the constraint value function of interest, i.e., $\mathcal{DS}_{\pi}$, cannot be directly obtained as the solution of fixed point equation. Thus in what follows, we will work with its approximation $\overline{\mathcal{DS}}_{\pi}(x)$, which is the fixed-point solution of 
\[
V(x)= T^{DS}_\pi[V](x), \,\,\,\forall x\in\mathcal X'.
\]
By definition, the following properties always hold: 
\[
\begin{split}
&\overline{\mathcal{DS}}_{\pi}(x)\geq d_0,\,\,\overline{\mathcal{DS}}_{\pi}(x)\geq\mathcal{DS}_{\pi}(x)\geq \mathcal D_\pi(x),\,\text{ for any $x\in \mathcal X'$.}
\end{split}
\]

To understand how this surrogate value function guarantees feasibility in problem $\mathcal{OPT}$,
consider the optimal value function 
\[
\overline{\mathcal{DS}}(x)=\min_{\pi\in\Delta(x)}  \overline{\mathcal{DS}}_{\pi}(x),\,\,\forall x\in\mathcal X',
\]
which is also the unique solution w.r.t. the fixed-point equation $\min_{\pi\in\Delta(x)}T^{DS}_\pi[V](x)=V(x)$. Now suppose at state $x_0$, the following condition holds: $\overline{\mathcal{DS}}(x_0)\leq d_0$. Then there exists a policy $\pi$ that is is safe w.r.t. problem $\mathcal{OPT}$, i.e., $\mathcal D_\pi(x_0)\leq\overline{\mathcal{DS}}_{\pi}(x_0)=d_0$.

Motivated by the above observation, we first check if the following condition holds: 
\[
\overline{\mathcal{DS}}(x_0)\leq d_0.
\]
If that is the case, define the set of feasible policies that is induced by the super-martingale $\overline{\mathcal{DS}}$ as
\[
\mathcal F^{\text{DS}}(x):=\!\big\{\pi\in\Delta:T^{DS}_\pi[\overline{\mathcal{DS}}](x)=\! \overline{\mathcal{DS}}(x)\!\big\},
\]
and solve the following problem, whose solution (if exists) is guaranteed to be safe.
\begin{quote} {\bf Problem $\mathcal{OPT}^{\text{DS}}$:} Assume $\overline{\mathcal{DS}}(x_0)\leq d_0$. Then given an initial state $x_0\in \mathcal X'$, and a threshold $d_0 \in \mathbb{R}_{\geq 0}$, solve
$\min_{\pi \in \bigcup_{x\in\mathcal X'}\mathcal F^{\text{DS}}(x)}\mathbb E\left[\sum_{t=0}^{\mathrm{T}^*-1}c(x_t, a_t)\mid x_0,\pi\right]$. 
\end{quote}
Similar to the step-wise approach, clearly the $\overline{\mathcal{DS}}$-induced Bellman operator 
\[
T_{\mathcal F^{\text{DS}}}[V](x)=\min_{\pi\in\mathcal F^{\text{DS}}(d_0,x)}T_{\pi,c}[V](x)
\]
is a contraction operator. This implies that there exists a unique solution $V^*_{\mathcal F^{\text{DS}}}$ to the fixed-point equation $T_{\mathcal F^{\text{DS}}}[V](x)=V(x)$, for $x\in\mathcal X'$, such that $V^*_{\mathcal F^{\text{DS}}}(x_0)$ is a solution to problem $\mathcal{OPT}^{\text{DS}}$. 

\newpage
\section{Proofs of the Technical Results in Section \ref{sec:lyapunov_cmdp}}

\subsection{Proof of Lemma \ref{lem:existence_L}}\label{appendix:lem:existence_L}
In the following part of the analysis we will use shorthand notation $P^*$ to denote the transition probability for 
$x\in\mathcal X'$ induced by the optimal policy, and $P_{B}$ to denote the transition probability for 
$x\in\mathcal X'$ induced by the baseline policy. These matrices are sub-stochastic  because we exclude the terms in the recurrent states. This means that both spectral radii $\rho(P^*)$ and $\rho(P_{B})$ are less than $1$, and thus both $(I-P^*)$ and $(I-P_{B})$ are invertible. By the Newmann series expansion, one can also show that
\[
(I-P^*)^{-1}=\left\{\sum_{t=0}^{\mathrm{T}^*-1}\mathbb P(x_t=x'|x_0=x,\pi^*)\right\}_{x,x'\in\mathcal X'},
\]
and
\[
(I-P_{B})^{-1}=\left\{\sum_{t=0}^{\mathrm{T}^*-1}\mathbb P(x_t=x'|x_0=x,\pi_B)\right\}_{x,x'\in\mathcal X'}.
\]
We also define $\Delta(a|x)=\pi_B(a|x)-\pi^*(a|x)$ for any $x\in\mathcal X'$ and $a\in\mathcal A$, and $P_\Delta=\{\sum_{a\in\mathcal A}P(x'|x,a)\Delta(a|x)\}_{x,x'\in\mathcal X'}$. Therefore, one can easily see that
\[
(I-P^*)(I-P_{B}+P_\Delta)^{-1}=I_{|\mathcal X'|\times|\mathcal X'|}.
\]
Therefore, by the Woodbury Sherman Morrison identity, we have that
\begin{equation}\label{eq:intermediate}
(I-P^*)^{-1}=(I-P_{B})^{-1}(I_{|\mathcal X'|\times|\mathcal X'|}+P_\Delta (I-P^*)^{-1}).
\end{equation}
By multiplying the constraint cost function vector $d(x)$ on both sides of the above equality, This further implies that for each $x\in\mathcal X'$, one has
\begin{equation}\label{eq:construction_lyap_opt}
\mathcal D_{\pi^*}(x)=\mathbb E\left[\sum_{t=0}^{\mathrm{T}^*-1}d(x_t)+\epsilon(x_t)\mid \pi_B,x\right]=L_\epsilon(x),
\end{equation}
such that
\[
T_{\pi^*}[L_\epsilon](x)=L_\epsilon(x),\,\,\forall x\in\mathcal X'.
\]
Here, the auxiliary constraint cost is given by
\begin{equation}\label{eq:eps_exact}
\epsilon(x)=\sum_{a\in\mathcal A}\Delta(a|x)\sum_{x'\in\mathcal X'}P(x'|x,a)\mathcal D_{\pi^*}(x').
\end{equation}
By construction, equation \eqref{eq:intermediate} immediately implies that $L_\epsilon$ is a fixed point solution of $T_{\pi^*}[V](x)=V(x)$ for $x\in\mathcal X'$. Furthermore, equation \eqref{eq:eps_exact} further implies that the upper bound of the constraint cost $\epsilon$ is given by:
\[
-2\overline{\mathrm{T}} D_{\max} D_{TV}(\pi^*||\pi_B)(x)\leq \epsilon(x)\leq 2\overline{\mathrm{T}} D_{\max} D_{TV}(\pi^*||\pi_B)(x),\,\,x\in\mathcal X',
\]
where $\overline{\mathrm{T}}$ is the uniform upper-bound of the MDP stopping time.

Since $\pi^*$ is also a feasible policy of problem $\mathcal{OPT}$, this further implies that $L_{\epsilon}(x_0)\leq d_0$.

\subsection{Proof of Theorem \ref{thm:L_star}}\label{appendix:thm:L_star}
First, under Assumption \ref{assumption:pi_b} the following inequality holds:
\[
D_{TV}(\pi^*||\pi_B)(x)\leq \frac{d_0-\mathcal D^{\pi_B}(x_0)}{2\overline{\mathrm{T}}^2 D_{\max}},\,\,\forall x\in\mathcal X'.
\]
Recall that $\epsilon^*(x)=2\overline{\mathrm{T}} D_{\max}D_{TV}(\pi^*||\pi_B)(x)$. The above expression implies that
\[
\mathbb E\left[\sum_{t=0}^{\mathrm{T}^*-1}\epsilon^*(x_t)|\pi_B,x_0\right]
\leq 2\overline{\mathrm{T}}^2 D_{\max}\max_{x\in\mathcal X'}D_{TV}(\pi^*||\pi_B)(x)\leq d_0-\mathcal D^{\pi_B}(x_0),
\]
which further implies 
\[
L_{\epsilon^*}(x_0)=\mathbb E\left[\sum_{t=0}^{\mathrm{T}^*-1}\epsilon^*(x_t)|\pi_B,x_0\right]+\mathcal D^{\pi_B}(x_0)\leq d_0,
\]
i.e., the second property in \eqref{eq:opt_lyap} holds.

Second, recall the following equality from \eqref{eq:construction_lyap_opt}:
\[
(I-P^*)(I-P_{B})^{-1}(\{d(x)\}_{x\in\mathcal X'}+\{\epsilon(x)\}_{x\in\mathcal X'})=\{d(x)\}_{x\in\mathcal X'}
\]
with the definition of the auxiliary constraint cost $\epsilon$ given by \eqref{eq:eps_exact}. We want to show that the first condition in \eqref{eq:opt_lyap} holds.
By adding the term $(I-P^*)(I-P_{B})^{-1}\epsilon^*$ to both sides of the above equality, it implies that:
\[
(I-P^*)(I-P_{B})^{-1}(D+\epsilon^*)=D+(I-P^*)(I-P_{B})^{-1}(\epsilon^*-\epsilon),
\]
where $\epsilon^*(x)-\epsilon(x)\geq 0$, for $x\in\mathcal X'$.
Therefore, the proof is completed if we can show that for any $x\in\mathcal X'$:
\begin{equation}\label{ineq:target_proof}
\{(I-P^*)(I-P_{B})^{-1}(\epsilon^*-\epsilon)\}(x)\geq 0.
\end{equation}
Now consider the following inequalities derived from Assumption \ref{assumption:pi_b}:
\[
\max_x D_{TV}(\pi^*||\pi_B)(x)\leq\frac{1}{2\overline{\mathrm{T}}}\frac{\overline{\mathrm{T}}D_{\max}-\overline{\mathcal D}}{\overline{\mathrm{T}}D_{\max}+\overline{\mathcal D}}\leq \frac{1}{2\overline{\mathrm{T}}}\frac{\overline{\mathrm{T}}D_{\max}-\overline{\mathcal D}^*}{\overline{\mathrm{T}}D_{\max}+\overline{\mathcal D}^*},
\]
the last inequality is due to the fact that $\overline{\mathcal D}\geq \overline{\mathcal D}^*$, where $\overline{\mathcal D}^*=\max_{x}\mathcal D_{\pi^*}(x)$ is the constraint upper-bound w.r.t. optimal policy $\pi^*$.
Multiplying the ratio $D_{TV}(\pi^*||\pi_B)(x)/\max_x D_{TV}(\pi^*||\pi_B)(x)\geq 0$ on both sides of the above inequality, for each $x\in\mathcal X'$ one obtains the following inequality:
\[
\begin{split}
D_{TV}(\pi^*||\pi_B)(x)\leq& \frac{(\overline{\mathrm{T}}D_{\max}-\overline{\mathcal D}^*)D_{TV}(\pi^*||\pi_B)(x)}{2 \overline{\mathrm{T}}(\overline{\mathrm{T}}D_{\max}+\overline{\mathcal D}^*)\max_x D_{TV}(\pi^*||\pi_B)(x)}\\
\leq& \frac{\epsilon^*(x)-\epsilon(x)}{2 \overline{\mathrm{T}}\max_{x}\{\epsilon^*(x)-\epsilon(x)\}},
\end{split}
\]
the last inequality holds due to the fact that for any $x\in\mathcal X'$,
\[
2 (\overline{\mathrm{T}}D_{\max}-\overline{\mathcal D}^*) D_{TV}(\pi^*||\pi_B)(x)
\leq\epsilon^*(x)-\epsilon(x)\leq 2 (\overline{\mathrm{T}}D_{\max}+\overline{\mathcal D}^*) D_{TV}(\pi^*||\pi_B)(x).
\]
Multiplying $2\overline{\mathrm{T}}\max_{x}\{\epsilon^*(x)-\epsilon(x)\}$ on both sides, it further implies that for each $x\in\mathcal X'$, one has the following inequality:
\[
(\epsilon^*(x)-\epsilon(x))-2\overline{\mathrm{T}}\max_{x'}\{\epsilon^*(x')-\epsilon(x')\} D_{TV}(\pi^*||\pi_B)(x)\geq 0.
\]
Now recall that $P^*=P_{B}-P_\Delta$, where $\Delta$ is equal to the matrix that characterizes the difference between the baseline and the optimal policy for each state in $\mathcal X'$ and action in $\mathcal A$, i.e., $\Delta(a|x)=\pi_B(a|x)-\pi^*(a|x)$, $\forall x\in\mathcal X'$, $\forall a\in\mathcal A$ and $P_\Delta=\{\sum_{a\in\mathcal A}P(x'|x,a)\Delta(a|x)\}_{x,x'\in\mathcal X'}$, the above condition guarantees that
\[
\{(I-P_{B}+P_\Delta)(I-P_{B})^{-1}(\epsilon^*-\epsilon)\}(x)\geq 0,\,\,\forall x\in\mathcal X'.
\]
This finally comes to the conclusion that under Condition $1$, the inequality in \eqref{ineq:target_proof} holds, which further implies that 
\[
(I-P^*)(I-P_{B})^{-1}(d(x)+\epsilon^*(x))\geq d(x),\,\,\forall x\in\mathcal X,
\]
i.e., the first property in \eqref{eq:opt_lyap} holds with 
$L_{\epsilon^*}(x)=\mathbb E\left[\sum_{t=0}^{\mathrm{T}^*-1}d(x)+\epsilon^*(x)|\pi_B,x\right]$.

By combining the above results, one shows that $L_{\epsilon^*}$ is a Lyapunov function that satisfies the properties in \eqref{eq:opt_lyap}  and \eqref{eq:baseline_lyap}, which concludes the proof.

\subsection{Properties of Safe Bellman Operator}\label{sec:appendix:safety_aware_bellman}
\begin{proposition}
The safe Bellman operator has the following properties.
\begin{itemize}
\item \textbf{Contraction}: There exists a vector with positive components, i.e., $\rho:\mathcal X\rightarrow\mathbb R_{\geq 0}$, and a discounting factor $0<\gamma<1$ such that 
\[
\|T[V]-T[W]\|_\rho\leq \gamma\|V-W\|_\rho,
\]
where the weighted norm is defined as $\|V\|_\rho=\max_{x\in\mathcal X} \frac{V(x)}{\rho(x)}$.
\item \textbf{Monotonicity}: For any value functions $V,\,W:\mathcal X\rightarrow\mathbb R$ such that $V(x)\leq W(x)$, one has the following inequality: $T[V](x)\leq T[W](x)$, for any state $x\in\mathcal X$.
\end{itemize}
\end{proposition}
\begin{proof}
First, we show the monotonicity property. For the case of $x\in\mathcal X\setminus\mathcal X'$, the property trivially holds. For the case of $x\in\mathcal X'$, given value functions $W,\,V:\mathcal X'\rightarrow \mathbb R$ such that $V(x)\leq W(x)$ for any $x\in\mathcal X'$, by the definition of Bellman operator $T$, one can show that for any $x\in\mathcal X'$ and any $a\in\mathcal A$,
\[
c(x,a)+\sum_{x'\in\mathcal X'}P(x'|x,a)V(x')
\leq
c(x,a)+\sum_{x'\in\mathcal X'}P(x'|x,a)W(x').
\]
Therefore, by multiplying $\pi(a|x)$ on both sides, summing the above expression over $a\in\mathcal A$, and taking the minimum of $\pi$ over the feasible set $\mathcal F_{L_{\epsilon^*}}(x)$, one can show that $T[V](x)\leq T[W](x)$ for any $x\in\mathcal X'$.

Second we show that the contraction property holds. For the case of $x\in\mathcal X\setminus\mathcal X'$, the property trivially holds. For the case of $x\in\mathcal X'$, following the construction in Proposition 3.3.1 of \cite{bertsekas1995dynamic}, consider a stochastic shortest path problem where the transition probabilities and the constraint cost function are the same as the one in problem $\mathcal{OPT}$, but the cost are all equal to $-1$. Then, there exists a fixed point value function $\hat V$, such that
\[
\hat V(x)=-1 + \min_{\pi\in \mathcal F_{L_{\epsilon^*}}(x)} \sum_{a}\pi(a|x)\sum_{x'\in\mathcal X'}P(x'|x,a)\hat V(x'),\,\,\forall x\in\mathcal X',
\]
such that the following inequality holds for given feasible Markovian policy $\pi'$:
\[
\hat V(x)
\leq -1 + \sum_{a}\pi'(a|x)\sum_{x'\in\mathcal X'}P(x'|x,a)\hat V(x'),\,\,\forall x\in\mathcal X'.
\]
Notice that $\hat V(x)\leq -1$ for all $x\in\mathcal X'$. By defining $\rho(x)=-\hat V(x)$, and by constructing
$
\gamma=\max_{x\in\mathcal X'} (\rho(x)-1)/\rho(x),
$
one immediately has $0<\gamma<1$, and
\[
\sum_{a}\pi'(a|x)\sum_{x'\in\mathcal X'}P(x'|x,a)\rho(x')\leq \rho(x) - 1\leq \gamma\rho(x),\,\,\forall x\in\mathcal X'.
\]
Then by using Proposition 1.5.2 of \cite{bertsekas1995dynamic}, one can show that $T$ is a contraction operator.
\end{proof}
\subsection{Proof of Theorem \ref{thm:safe_opt}}\label{appendix:thm:safe_opt}
Let $V_{\mathcal{OPT}}(x_0)$ be the optimal value function of problem $\mathcal{OPT}$, and let $V^*$ be a fixed point solution: $V(x)=T[V](x)$, for any $x\in\mathcal X$.
For the case when $x_0\in\mathcal X\setminus\mathcal X'$, the following result trivially holds: $V_{\mathcal{OPT}}(x_0)=T[V_{\mathcal{OPT}}](x_0)=V^*(x_0)=0$. Below, we show the equality for the case of $x_0\in\mathcal X'$.

First, we want to show that $V_{\mathcal{OPT}}(x_0)\leq V^*(x_0)$. Consider the greedy policy $\overline\pi^*$ constructed from the fixed point equation. Immediately, one has that $\overline\pi^*(\cdot|x)\in\mathcal F_{L_{\epsilon^*}}(x)$. This implies 
\begin{equation}\label{eq:lyap_ineq_opt}
T_{\overline\pi^*,d}[L_{\epsilon^*}](x)\!\leq\! L_{\epsilon^*}(x),\,\,\forall x\in\mathcal X'.
\end{equation}
Thus by recursively applying $T_{\overline\pi^*,d}$ on both sides of the above inequality, the contraction property of Bellman operator $T_{\overline\pi^*,d}$ implies that one has the following expression:
\[
\lim_{n\rightarrow\infty}T^n_{\overline\pi^*,d}[L_{\epsilon^*}](x_0)=\mathbb E\left[\sum_{t=0}^{\mathrm{T}^*-1} d(x_t)+L_{\epsilon^*}(x_{\mathrm{T}^*})\mid x_0,\overline\pi^*\right]\leq L_{\epsilon^*}(x_0)\leq d_0.
\]
Since the state enters the terminal set at time $t=\mathrm{T}^*$, we have that $L_{\epsilon^*}(x_{\mathrm{T}^*})=0$ almost surely. Then the above inequality implies
$
\mathbb E\left[\sum_{t=0}^{\mathrm{T}^*-1} d(x_t)\mid x_0,\overline\pi^*\right]\leq d_0,
$
which further shows that $\overline\pi^*$ is a feasible policy to problem $\mathcal{OPT}$. On the other hand, recall that $V^*(x)$ is a fixed point solution to $V(x)=T[V](x)$, for any $x\in\mathcal X'$. 
Then for any bounded initial value function $V_0$, the contraction property of Bellman operator $T_{\overline\pi^*,c}$ implies that 
\[
V^*(x)=\lim_{n\rightarrow\infty}T^n_{\overline\pi^*,c}[V_0](x)=\lim_{n\rightarrow\infty}\mathbb E\left[\sum_{t=0}^{n-1} c(x_t,a_t)+V_0(x_{n})\mid x_0=x,\overline\pi^*\right],
\]
for which the transient assumption of stopping MDPs further implies that
\[
V^*(x)=\mathbb E\left[\sum_{t=0}^{\mathrm{T}^*-1} c(x_t,a_t)\mid x_0=x,\overline\pi^*\right].
\]
Since $\overline\pi^*$ is a feasible solution to problem $\mathcal{OPT}$. This further implies that $V_{\mathcal{OPT}}(x_0)\leq V^*(x_0)$.

Second, we want to show that $V_{\mathcal{OPT}}(x_0)\geq V^*(x_0)$. Consider the optimal policy $\pi^*$ of problem $\mathcal{OPT}$ that is used to construct Lyapunov function $L_{\epsilon^*}$. Since the Lyapunov fnction satisfies the following Bellman inequality:
\[
T_{\pi^*,d}[L_{\epsilon^*}](x)\leq L_{\epsilon^*}(x),\,\,\forall x\in\mathcal X',
\]
it implies that the optimal policy $\pi^*$ is a feasible solution to the optimization problem in Bellman operator $T[V^*](x)$. Note that $V^*$ is a fixed point solution to equation: $V^*(x)=T[V^*](x)$, for any $x\in\mathcal X'$. Immediately the above result yields the following inequality:
\[
V^*(x)=T_{\overline\pi^*,c}[V^*](x)\leq T_{\pi^*,c}[V^*](x),\,\,\forall x\in\mathcal X',
\]
the first equality holds because $\overline\pi^*(\cdot|x)$ is the minimizer of the optimization problem in $T[V^*](x)$, $x\in\mathcal X'$.
By recursively applying  Bellman operator $T_{\pi^*,c}$ to this inequality, one has the following result:
\[
V^*(x)\leq \lim_{n\rightarrow\infty}T^n_{\pi^*,c}[V^*](x)=\mathcal C_{\pi^*}(x)=V_{\mathcal{OPT}}(x),\,\,\forall x\in\mathcal X'.
\]
One thus concludes that $V_{\mathcal{OPT}}(x_0)\geq V^*(x_0)$.

Combining the above analysis, we prove the claim of $V_{\mathcal{OPT}}(x_0)=V^*(x_0)$, and the greedy policy of the fixed-point equation, i.e., $\overline\pi^*$, is an optimal policy to problem $\mathcal{OPT}$.

\subsection{Proof of Proposition \ref{prop:properties_safe_PI}}\label{appendix:properties_safe_PI}
For the derivations of consistent feasibility and policy improvement, without loss of generality we only consider the case of $k=0$.

To show the property of consistent feasibility, consider an arbitrary feasible policy $\pi_0$ of problem $\mathcal {OPT}$. By definition, one has $\mathcal D_{\pi_0}(x_0)\leq d_0$, and the value function $\mathcal D_{\pi_0}$ has the following property:
\[
\sum_a\pi_0(a|x)\left[d(x)+\sum_{x'}P(x'|x,a)\mathcal D_{\pi_0}(x')\right]=\mathcal D_{\pi_0}(x),\,\,\forall x\in\mathcal X'.
\]
Immediately, since $\mathcal D_{\pi_0}$ satisfies the constraint in \eqref{eq:opt_eps_baseline}, one can treat it as a Lyapunov function, this shows that the set of Lyapunov functions $\mathcal L_{\pi_0}(x_0,d_0)$ is non-empty. 
Therefore, there exists a bounded Lyapunov function $\{L_{\epsilon_0}(x)\}_{x\in\mathcal X}$ as the solution of the optimization problem in Step 0. Now consider the policy optimization problem in Step 1. Based on the construction of $\{L_{\epsilon_0}(x)\}_{x\in\mathcal X}$, the current policy $\pi_0$ is a feasible solution to this problem, therefore the feasibility set is non-empty. Furthermore, by recursively applying the inequality constraint on the updated policy $\pi_1$ for $\mathrm{T}^*-1$ times, one has the following inequality: 
\[
\mathbb E\left[\sum_{t=0}^{\mathrm{T}^*-1}d(x_t)+L_{\epsilon_0}(x_{\mathrm{T}^*})|x_0,\pi_1\right]\leq L_{\epsilon_0}(x_0)\leq d_0.
\]
This shows that $\pi_1$ is a feasible policy to problem $\mathcal{OPT}$. 

To show the property of policy improvement, consider the policy optimization in Step 1. Notice that the current policy $\pi_0$ is a feasible solution of this problem (with Lyapunov function $L_0$), and the updated policy $\pi_1$ is a minimizer of this problem. Then, one immediately has the following chain of inequalities:
\[
\begin{split}
T_{\pi_1,c}[V_0](x)&=\sum_{a\in\mathcal A}\pi_1(a|x)\left[c(x,a)\!+\!\sum_{x'\in\mathcal X'}P(x'|x,a)V_0(x')\right]\\
&\leq\sum_{a\in\mathcal A}\pi_0(a|x)\left[c(x,a)\!+\!\sum_{x'\in\mathcal X'}P(x'|x,a)V_0(x')\right]= V_0(x),\,\,\forall x\in\mathcal X',
\end{split}
\]
where the last equality is due to the fact that $V_0(x)=\mathcal C_{\pi_0}(x)$, for any $x\in\mathcal X$.
By the contraction property of Bellman operator $T_{\pi_1}$, the above condition further implies 
\[
\mathcal C_{\pi_1}(x) = \lim_{n\rightarrow\infty}T^n_{\pi_1,c}[V_0](x)\leq V_0(x)=\mathcal C_{\pi_0}(x), \,\forall x\in\mathcal X',
\]
which proves the claim about policy improvement.

To show the property of asymptotic convergence, notice that the value function sequence $\{\mathcal C_{\pi_{k}}(\cdot)\}_{k\geq 0}$ is uniformly monotonic, and each element is uniformly lower bounded by the unique solution of fixed point equation: $V(x)=\min_{a\in\mathcal A} c(x,a)+\sum_{x'\in\mathcal X'}P(x'|x,a)V(x')$, $\forall x\in\mathcal X'$. Therefore, this sequence of value function converges (point-wise) as soon as in the limit the policy improvement stops. 
Whenever this happens, i.e., there exists $K\geq 0$ such that $\max_{x\in\mathcal X'}|\mathcal C_{\pi_{K+1}}(x)- \mathcal C_{\pi_K}(x)|\leq\epsilon$ for any $\epsilon>0$, then this value function is the fixed point of $\min_{\pi\in\mathcal F_{ L_K}(x)}T_{\pi,c}[V](x)=V(x)$, $\forall x\in\mathcal X'$, whose solution policy is unique (due to the strict convexity of the objective function in the policy optimization problem after adding a convex regularizer). Furthermore, due to the strict concavity of the objective function in problem in \eqref{eq:opt_eps_baseline} (after adding a concave regularizer), the solution pair of this problem is unique, which means the update of $\{( L_{\epsilon_k},\epsilon_{k})\}$ stops at step $K$. Together, this also means that policy update $\{\pi_k\}$ converges.

\subsection{Analysis on Performance Improvement in Safe Policy Iteration}\label{appendix:SPI_performance}
Similar to the analysis in \cite{achiam2017constrained}, the following lemma provides a bound in policy improvement.
\begin{lemma}\label{lem:perf_improve}
For any policies $\pi'$ and $\pi$, define the following error function:
\[
\small
\begin{split}
\Lambda (\pi,\pi')=&\mathbb E\left[\sum_{t=0}^{\mathrm{T}^*-1}\left(\frac{\pi'(a_t|x_t)}{\pi(a_t|x_t)}-1\right)\underbrace{\left(C(x_t,a_t)+V^\pi(x_{t+1})-V^\pi(x_t)\right)}_{\delta^\pi(x_t,a_t,x_{t+1})}\mid x_0,\pi\right]\\
=&\mathbb E\left[\sum_{t=0}^{\mathrm{T}^*-1}\mathbb E_{a\sim\pi'(\cdot|x)}[Q^\pi(x_t,a)]-V^\pi(x_t)\mid x_0,\pi\right]
\end{split}
\]
and $\Delta^{\pi'}=\max_{x\in\mathcal X'}|\mathbb E_{a\sim\pi'(\cdot|x)}[Q^\pi(x,a)]-V^\pi(x)|$. Then, the following error bound on the performance difference between $\pi$ and $\pi'$ holds:
\[
\Lambda (\pi,\pi')-\mathcal E^{\pi,\pi'}_{TV}\leq\mathcal C_{\pi'}(x_0)-\mathcal C_{\pi}(x_0)\leq \Lambda (\pi,\pi')+\mathcal E^{\pi,\pi'}_{TV}.
\]
where $\mathcal E^{\pi,\pi'}_{TV}=2\Delta^{\pi'}\cdot\left(\max_{x_0\in\mathcal X'}\mathbb E[\mathrm{T}^*|x_0,\pi']\right)\cdot\mathbb E\left[\sum_{t=0}^{\mathrm{T}^*-1}D_{TV}(\pi'||\pi)(x_t)\mid x_0,\pi\right]$.
\end{lemma}
\begin{proof}
First, it is clear from the property of telescopic sum that
\[
\small
\begin{split}
&\mathcal C_{\pi'}(x_0)-\mathcal C_{\pi}(x_0)=\mathbb E\left[\sum_{t=0}^{\mathrm{T}^*-1}\delta^{\pi'}(x_t,a_t,x_{t+1})\mid x_0,\pi'\right]-\mathbb E\left[\sum_{t=0}^{\mathrm{T}^*-1}\delta^\pi(x_t,a_t,x_{t+1})\mid x_0,\pi\right]\\
\leq&\mathbb E\left[\sum_{t=0}^{\mathrm{T}^*-1}\delta^{\pi'}(x_t,a_t,x_{t+1})-\delta^\pi(x_t,a_t,x_{t+1})\mid x_0,\pi\right]+\Delta^{\pi'}\cdot\sum_{y\in\mathcal X'}\left|\sum_{t=0}^{\mathrm{T}^*-1}\mathbb P(x_t=y|x_0,\pi')-\sum_{t=0}^{\mathrm{T}^*-1}\mathbb P(x_t=y|x_0,\pi)\right|,
\end{split}
\]
where the inequality is based on the Holder inequality $\mathbb E[|xy|]\leq \left(\mathbb E[|x|^p]\right)^{1/p}\left(\mathbb E[|y|^q] \right)^{1/q}$ with $p=1$ and $q=\infty$. 

Immediately, the first part of the above expression is re-written as: 
$\mathbb E\left[\sum_{t=0}^{\mathrm{T}^*-1}\delta^{\pi'}(x_t,a_t,x_{t+1})-\delta^\pi(x_t,a_t,x_{t+1})\mid x_0,\pi\right]=\Lambda(\pi,\pi')$. Recall that shorthand notation $P_\pi$ to denote the transition probability for 
$x\in\mathcal X'$ induced by the policy $\pi$. For the second part of the above expression, notice that the following chain of inequalities holds:
\[
\begin{split}
&\sum_{y\in\mathcal X'}\left|\sum_{t=0}^{\mathrm{T}^*-1}\mathbb P(x_t=y|x_0,\pi)-\sum_{t=0}^{\mathrm{T}^*-1}\mathbb P(x_t=y|x_0,\pi')\right|\\
=&\sum_{y\in\mathcal X'}\left|\mathbf 1(x_0)^\top\left((I-P_{\pi})^{-1} -(I-P_{\pi'})^{-1}\right)\mathbf 1(y)\right|\\
=&\sum_{y\in\mathcal X'}\left|\mathbf 1(x_0)^\top(I-P_{\pi})^{-1}(P_{\pi}-P_{\pi'})(I-P_{\pi'})^{-1}\mathbf 1(y)\right|\\
\leq&\|(\mathbf 1(x_0)^\top(I-P_{\pi})^{-1}(P_{\pi}-P_{\pi'})\|_1\cdot\sum_{y\in\mathcal X'}\max_{x_0\in\mathcal X'}\mathbf 1(x_0)(I-P_{\pi'})^{-1}\mathbf 1(y)\\
\leq&\sum_{y\in\mathcal X'}\overline{\mathrm{T}}\cdot\mathbf 1(y)\cdot\mathbb E\left[\sum_{t=0}^{\mathrm{T}^*-1}\sum_{x'\in\mathcal X'}P(x'|x_t,a_t)\sum_{a\in\mathcal A}\left|\pi(a_t|x_t)-\pi'(a_t|x_t)\right|\mid x_0,\pi\right]\\
=&2\overline{\mathrm{T}}\cdot\mathbb E\left[\sum_{t=0}^{\mathrm{T}^*-1}D_{TV}(\pi'||\pi)(x_t)\mid x_0,\pi\right],
\end{split}
\]
the first, is based on the Holder inequality with $p=1$ and $q=\infty$ and on the fact that all entries in $(I-P_{\pi'})^{-1}$ is non-negative, the second inequality is due to the fact that starting at any initial state $x_0$, it almost takes $\overline{\mathrm{T}}$ steps to the set of recurrent states $\mathcal X\setminus\mathcal X'$. In other words, one has the following inequality:
\[
\mathbf 1(x_0)(I-P_{\pi'})^{-1}\mathbf 1(y)=\mathbb E\left[\sum_{t=0}^{\mathrm{T}^*-1}\mathbf 1\{x_t=y\}\mid x_0,\pi'\right]\leq \overline{\mathrm{T}},\,\,\forall x_0,y\in\mathcal X'.
\]
Therefore, combining with these properties the proof of the above error bound is completed.
\end{proof}
Using this result, the sub-optimality performance bound of policy $\pi_{k^*}$ from SPI is $\overline{\mathrm{T}}C_{\max}-\big(\sum_{k=0}^{k^*-1}\max\{0, \Lambda^{\pi_k,\pi_{k+1}}-\mathcal E^{\pi_k,\pi_{k+1}}_{TV}\}+\mathcal C_{\pi_0}(x_0)\big)$. 

\subsection{Proof of Proposition \ref{prop:properties_safe_VI}}\label{appendix:properties_safe_VI}
For the derivations of consistent feasibility and monotonic improvement on value estimation, without loss of generality we only consider the case of $t=0$.

To show the property of consistent feasibility, notice that with the definitions of the initial $Q-$function $Q_0$, the initial Lyapunov function $L_{\epsilon_0}$ w.r.t. the initial auxiliary cost $\epsilon_0$, the corresponding induced policy $\pi_0$ is feasible to problem $\mathcal {OPT}$, i.e., $\mathcal D_{\pi_0}(x_0)\leq d_0$. Consider the optimization problem in Step 1. 
Immediately, since $\mathcal D_{\pi_0}$ satisfies the constraint in \eqref{eq:opt_eps_baseline}, one can treat it as a Lyapunov function, and the set of Lyapunov functions $\mathcal L_{\pi_0}(x_0,d_0)$ is non-empty. 
Therefore, there exists a bounded Lyapunov function $\{L_{\epsilon_1}(x)\}_{x\in\mathcal X}$ and auxiliary cost  $\{\epsilon_1(x)\}_{x\in\mathcal X}$ as the solution of this optimization problem. Now consider the policy update in Step 0. Since $\pi_1(\cdot|\cdot)$ belongs to the set of feasible policies $\mathcal F_{L_{\epsilon_1}}(\cdot)$, by recursively applying the inequality constraint on the updated policy $\pi_1$ for $\mathrm{T}^*-1$ times, one has the following inequality: 
\[
\mathbb E\left[\sum_{t=0}^{\mathrm{T}^*-1}d(x_t)+L_{\epsilon_1}(x_{\mathrm{T}^*})|x_0,\pi_1\right]\leq L_{\epsilon_1}(x_0)\leq d_0.
\]
This shows that $\pi_1$ is a feasible policy to problem $\mathcal{OPT}$.

To show the asymptotic convergence property, for every initial state $x_0$ and any time step $K$, with the policy $\pi=\{\pi_0,\ldots,\pi_{K-1}\}$ generated by the value iteration procedure, the cumulative cost can be broken down into the following two portions, which consists of the cost over the first $K$ stages and the remaining cost. Specifically,
\[
\overline V(x_0)=\mathbb E\left[\sum_{t=0}^{\mathrm{T}^*-1}c(x_t,a_t)\mid x_0,\pi\right]=\mathbb E\left[\sum_{t=0}^{K-1}c(x_t,a_t)\mid x_0,\pi\right]+\mathbb E\left[\sum_{t=K}^{\mathrm{T}^*-1}c(x_t,a_t)\mid x_0,\pi\right],
\]
where the second term is bounded $\mathbb E[\mathrm{T}^*-K|\pi,x_0]C_{\max}$, which is bounded by $\sum_{t=K}^\infty \mathbb P(x_t\in\mathcal X'\mid x_0,\pi)\cdot C_{\max}>0$. Since the value function $V_0(x)=\min_{\pi\in\mathcal F_{L_{\epsilon_0}}(x)}\pi(\cdot|x)^\top Q_0(x,\cdot)$ is also bounded, one can further show the following inequality:
\[
\begin{split}
-\mathbb P(x_K\in\mathcal X'\mid x_0,\pi)\cdot\|V_0\|_\infty\leq&  T_{K-1}[\cdots[T_0[V_0]]\cdots](x_0)-\mathbb E\left[\sum_{t=0}^{K-1}c(x_t,a_t)\mid x_0,\pi\right]\\
\leq& \mathbb P(x_K\in\mathcal X'\mid x_0,\pi)\cdot\|V_0\|_\infty.
\end{split}
\]
Recall from our problem setting that all policies are proper (see Assumption 3.1.1 and Assumption 3.1.2 in \cite{bertsekas1995dynamic}).  Then by the property of a transient MDP (see Definition 7.1 in \cite{altman1999constrained}), the sum of probabilities of the state trajectory after step $K$ that is in the transient set $\mathcal X'$, i.e., $\sum_{t=K}^\infty \mathbb P(x_t\in\mathcal X'\mid x_0,\pi)$, is bounded by $M_{\pi}\cdot\epsilon$.
Therefore, as $K$ goes to $\infty$, $\epsilon$ approaches $0$. Using the result that $\sum_{t=K}^\infty \mathbb P(x_t\in\mathcal X'\mid x_0,\pi)$ vanishes as $K$ goes to $\infty$, one concludes that
\[
\lim_{K\rightarrow\infty}T_{K-1}[\cdots[T_0[V_0]]\cdots](x_0)=\overline V(x_0),\,\,\forall x_0\in\mathcal X',
\]
which completes the proof of this property.

\newpage
\section{Pseudo-code of Safe Reinforcement Learning Algorithms}\label{appendix:pseudo_code_safe_RL}
\begin{algorithm}
\caption{Policy Distillation}
\label{alg:policy_distill}    
\begin{algorithmic}
\STATE {\bfseries Input:} 
Policy parameterization $\pi_\phi$ with parameter $\phi$; A batch of state trajectories $\{x_{0,m},\ldots,x_{\overline{\mathrm{T}}-1,m}\}_{m=1}^M$ generated by following the baseline policy $\pi_B$
\STATE Compute the action probabilities $\{\pi'(\cdot|x_{0,m}),\ldots,\pi'(\cdot|x_{\overline{\mathrm{T}}-1,m})\}_{m=1}^M$ by solving problem in \eqref{}.
\STATE Compute the policy parameter by supervised learning:
\[
\phi^*\in\arg\min_{\phi}\frac{1}{m}\sum_{m=1}^M\sum_{t=0}^{\overline{\mathrm{T}}-1} D_{\text{JSD}}(\pi_\phi(\cdot|x_{t,m})\parallel\pi'(\cdot|x_{t,m}))
\]
where $D_{\text{JSD}}(P|| Q)=\frac  {1}{2}D(P\parallel \frac {1}{2}(P+Q))+\frac  {1}{2}D(Q\parallel  \frac {1}{2}(P+Q))$ is the Jensen-Shannon divergence
\STATE {\bf Return} Distilled policy $\pi_{\phi^*}$
\end{algorithmic}
\end{algorithm}

\begin{algorithm}
\caption{Safe DQN}
\label{alg:safe_dqn}    
\begin{algorithmic}
\STATE {\bfseries Input:} 
Initial prioritized replay buffer $M = \{\emptyset\}$; Initial importance weights $w_0=1$, $w_{D,0}=1$ $w_{T,0}=1$; Mini-batch size $|B|$; Network parameters $\theta^-$, $\theta^-_D$, $\theta^-_T$; Initial feasible policy $\pi_0$;
\FOR{$k\in\{0,1,\ldots,\}$}
\FOR{$t=0$ {\bfseries to} $\overline{\mathrm{T}}-1$}
\STATE Sample action $a_t \sim \pi_k(\cdot|x_t)$, execute $a_t$ and observe next state $x_{t+1}$, and costs $(c_t,d_t)$ \STATE Add experiences to replay buffer, i.e., $M \leftarrow (x_t, a_t, c_t, d_t,x_{t+1}, w_0, w_{D,0}, w_{T,0}) \cup M$
\STATE From the buffer $M$, sample a mini-batch $B=\{(x_j , a_j , c_j , d_j, x'_{j}, w_j, w_{D,j}, w_{T,j})\}_{j=1}^{|B|}$ and set the targets $y_{D,j}$, $y_{T,j}$, and $y_{j}$
\STATE Update the $\theta$ parameters as follows: 
\[
\small
\begin{split}
&\theta_D \leftarrow\theta_D^--\kappa_j\cdot w_{D,j} \cdot (y_{D,j}-\widehat Q_{D}(x_j,a_j;\theta_D^-))\cdot\nabla_\theta \widehat Q_{D}(x_j,a_j;\theta)|_{\theta=\theta_D^-},\\ 
&\theta_T \leftarrow\theta_T^--\kappa_j\cdot w_{T,j}\cdot (y_{T,j}-\widehat Q_{T}(x_j,a_j;\theta_T^-))\cdot\nabla_\theta \widehat Q_{T}(x_j,a_j;\theta)|_{\theta=\theta_T^-},\\ 
&\theta \leftarrow\theta^--\eta_j\cdot w_j\cdot (y_j-\widehat Q(x_j,a_j;\theta^-))\cdot\nabla_\theta \widehat Q(x_j,a_j;\theta)|_{\theta=\theta^-}
\end{split}
\]
where the target values are respectively $y_{D,j} = d(x_j)+\pi_k(\cdot|x_j')^\top \widehat Q_{D}(x_j',\cdot;\theta_D)$, $y_{T,j} = \mathbf 1\{x_j\in\mathcal X'\}+\pi_k(\cdot|x_j')^\top \widehat Q_{T}(x_j',\cdot;\theta_T)$, and $y_j = c(x_j,a_j)+\pi'(\cdot|x_j')^\top  \widehat Q(x_j',\cdot;\theta)$, and $\pi'(\cdot|x'_j)$ is the greedy action probability w.r.t. $x'_j$, which is a solution to \eqref{eq:approx_greedy_pol}
\STATE \underline{Prioritized Sweep}: Update importance weights $w_j$, $w_{D,j}$, and $w_{T,j}$ of the samples in mini-batch $B$, based on TD errors $|y_j-\widehat Q(x_j,a_j;\theta)|$, $|y_{D,j}-\widehat Q_{D}(x_j,a_j;\theta_D)|$ and $|y_{T,j}-\widehat Q_{T}(x_j,a_j;\theta_T)|$
\STATE \underline{Distillation}: Update the policy to $\pi_{k+1}$ using Algorithm \ref{alg:policy_distill} w.r.t. data $\{x'_{0,j},\ldots,x'_{\overline{\mathrm{T}}-1,j}\}_{j=1}^{|B|}$ and $\{\pi'(\cdot|x'_{0,j}),\ldots,\pi'(\cdot|x'_{\overline{\mathrm{T}}-1,j)}\}_{j=1}^{|B|}$
\ENDFOR
\STATE \underline{Double $Q-$learning}: Update $\theta^- = \theta$, $\theta^-_D = \theta_D$ and $\theta^-_T = \theta_T$ after $c$ iterations
\ENDFOR
\STATE {\bf Return} 
\end{algorithmic}
\end{algorithm}

\begin{algorithm}
\caption{Safe DPI}
\label{alg:safe_cpi}    
\begin{algorithmic}
\STATE {\bfseries Input:} 
Initial prioritized replay buffer $M = \{\emptyset\}$; Initial importance weights $w_0=1$, $w_{D,0}=1$ $w_{T,0}=1$; Mini-batch size $|B|$;  Initial feasible policy $\pi_0$;
\FOR{$k\in\{0,1,\ldots,\}$}
\STATE Sample action $a_t \sim \pi_k(\cdot|x_t)$, execute $a_t$ and observe next state $x_{t+1}$, and costs $(c_t,d_t)$ \STATE Add experiences to replay buffer, i.e., $M \leftarrow (x_t, a_t, c_t, d_t,x_{t+1}, w_0, w_{D,0}, w_{T,0}) \cup M$
\STATE From the buffer $M$, sample a mini-batch $B=\{(x_j , a_j , c_j , d_j, x'_{j}, w_j, w_{D,j}, w_{T,j})\}_{j=1}^{|B|}$ and set the targets $y_{D,j}$, $y_{T,j}$, and $y_{j}$
\STATE Update the $\theta$ parameters as follows: 
\[
\small
\begin{split}
\theta^*_{\pi_k}\in&\arg\min_{\theta}\sum_{t,j}\big(c_{t,j}-\widehat Q(x_{t,j},a_{t,j};\theta) + \pi_k(\cdot|x_{t+1,j})^\top \widehat Q(x_{t+1,j},\cdot;\theta)\big)^2,\\
\theta^*_{D,\pi_k}\in&\arg\min_{\theta_D}\sum_{t,j}\big(d_{t,j}-\widehat Q_D(x_{t,j},a_{t,j};\theta_D) + \pi_k(\cdot|x_{t+1,j})^\top \widehat Q_D(x_{t+1,j},\cdot;\theta_D)\big)^2,\\
\theta^*_{T,\pi_k}\in&\arg\!\min_{\theta_T}\sum_{t,j}\big(\mathbf 1\{x_{t,j}\!\!\in\!\mathcal X'\}\!-\!Q_T(x_{t,j}\!,a_{t,j};\theta_T) + \pi_k(\cdot|x_{t+1,j})^\top Q_T(x_{t+1,j},\cdot;\theta_T)\big)^2
\end{split}
\]
and construct the $Q-$functions $\widehat Q(x, a;\theta^*_{\pi_k})$ and $\widehat Q_L(x, a;\theta^*_{D,\pi_k},\theta^*_{T,\pi_k})=\widehat Q_{D}(x,a;\theta_D)+\widetilde\epsilon'\cdot \widehat Q_{T}(x,a;\theta_T)$
\STATE Calculate greedy action probabilities $\{\pi'(\cdot|x_{0,j}),\ldots,\pi'(\cdot|x_{\overline{\mathrm{T}}-1,j)}\}_{j=1}^{|B|}$ by solving problem \eqref{eq:approx_greedy_pol}, with respect to batch of states drawn from the replay buffer $\{x_{0,j},\ldots,x_{\overline{\mathrm{T}}-1,j}\}_{j=1}^{|B|}$ 
\STATE \underline{Distillation}: Update the policy to $\pi_{k+1}$ using Algorithm \ref{alg:policy_distill} w.r.t. data $\{x_{0,j},\ldots,x_{\overline{\mathrm{T}}-1,j}\}_{j=1}^{|B|}$ and $\{\pi'(\cdot|x_{0,j}),\ldots,\pi'(\cdot|x_{\overline{\mathrm{T}}-1,j)}\}_{j=1}^{|B|}$
\ENDFOR
\STATE {\bf Return} Final policy $\pi_{k^*}$
\end{algorithmic}
\end{algorithm}

\newpage
\section{Practical Implementations}\label{appendix:practical}
There are several techniques that improve training and scalability of the safe RL algorithms.
To improve stability in training $Q$ networks, one may apply double $Q-$learning \citep{hasselt2010double} to separate the target values and the value function parameters and to slowly update the target $Q$ values at every predetermined iterations. 
On the other hand, to incentivize learning at state-action pairs that have high temporal difference (TD) errors, one can use a prioritized sweep in replay buffers \citep{schaul2015prioritized} to add an importance weight to relevant experience. To extend the safe RL algorithms to handle continuous actions, one may adopt the normalized advantage functions (NAFs) \citep{gu2016continuous} parameterization for $Q-$functions. Finally, instead of exactly solving the LP problem for policy optimization in \eqref{eq:approx_greedy_pol}, one may approximate this solution by solving its entropy regularized counterpart 
\citep{neu2017unified}. This approximation has an elegant closed-form solution that is parameterized by a Lagrange multiplier, which can be effectively computed by binary search methods (see Appendix \ref{appendix:inner_opt_discrete} and Appendix \ref{appendix:inner_opt_cont} for details).

\subsection{Case 1: Discrete Action Space}\label{appendix:inner_opt_discrete}
In this case, problem \eqref{eq:approx_greedy_pol} is cast as finite dimensional linear programming (LP). In order to effectively approximate the solution policy especially when the action space becomes large, instead of exactly solving this inner optimization problem, one considers its Shannon entropy-regularized variant:
\begin{align}
\min_{\pi\in\Delta}&\pi(\cdot|x)^\top  (Q(x,\cdot)+\tau \log\pi(\cdot|x))\label{eq:approx_greedy_pol_entropy}\\
\text{s.t. }&(\pi(\cdot|x)-\pi_B(\cdot|x))^\top Q_L(x, \cdot)\leq \widetilde\epsilon'(x)\nonumber
\end{align}
where $\tau>0$ is the regularization constant. When $\tau\rightarrow 0$, then $\pi^*_{\tau}$ converges to the original solution policy $\pi^*$.

We will hereby illustrate how to effectively solve $\pi^*_{\tau}$ for any given $\tau>0$ without explicitly solving the LP. Consider the Lagrangian problem for entropy-regularized optimization:
\[
\min_{\pi\in\Delta}\!\!\max_{\lambda\geq 0}\Gamma_x(\pi,\lambda),
\]
where 
$
\Gamma_x(\pi,\lambda)=\pi(\cdot|x)^\top \!(Q(x,\!\cdot)+\lambda Q_L(x,\!\cdot)+\tau \log\pi(\cdot|x)) - \lambda (\pi_B(\cdot|x)^\top Q_L(x, \cdot)+\widetilde\epsilon'(x))
$
is the Lagrangian function.
Notice that the set of stationary Markovian policies $\Delta$ is a convex set, and the objective function is a convex function in $\pi$ and concave in $\lambda$. By strong duality, there exists a saddle-point to the Lagrangian problem where solution policy is equal to $\pi^*_{\tau}$, and it can be solved by the maximin problem:
\[
\max_{\lambda\geq 0}\!\min_{\pi\in\Delta}\Gamma_x(\pi,\lambda).
\]
For the inner minimization problem, it has been shown that the $\lambda-$solution policy has the following closed form:
\[
\pi^*_{\tau,\lambda}(\cdot|x)\propto\exp\left(-\frac{Q(x,\!\cdot)+\lambda Q_L(x,\!\cdot)}{\tau}\right).
\]
Equipped with this formulation, we now solve the problem for the optimal Lagrange multiplier $\lambda^*(x)$ at state $x\in\mathcal X'$:
\[
\max_{\lambda\geq 0}-\tau\cdot\text{logsumexp}\left(-\frac{Q(x,\cdot)+\lambda Q_L(x,\cdot)}{\tau}\right)- \lambda (\pi_B(\cdot|x)^\top Q_L(x, \cdot)+\widetilde\epsilon'(x)),
\]
where $\text{logsumexp}(y)=\log\sum_a\exp(y_a)$ is a strictly convex function in $y$, and the objective function is a concave function of $\lambda$. Notice that this problem has a unique optimal Lagrange multiplier that is the solution of the following KKT condition:
\[
\pi_B(\cdot|x)^\top Q_L(x, \cdot)+\widetilde\epsilon'(x)+\frac{\sum_{a}Q_L(x,a)\exp\left(-\frac{Q(x,a)+\lambda Q_L(x,a)}{\tau}\right)}{\sum_{a}\exp\left(-\frac{Q(x,a)+\lambda Q_L(x,a)}{\tau}\right)}=0.
\]
Using the parameterization $z=\exp(-\lambda)$, this condition can be written as the following polynomial equation in $z$:
\begin{equation}\label{eq:lambda}
\sum_{a}\left(Q_L(x,a)+\pi_B(\cdot|x)^\top Q_L(x, \cdot)+\widetilde\epsilon'(x)\right)\cdot\exp\left(-\frac{Q(x,a)}{\tau}\right)\cdot z^{\frac{Q_L(x,a)}{\tau}}=0.
\end{equation}
Therefore, the solution $0\leq z^*(x)\leq 1$ can be solved by computing the root solution of the above polynomial and the optimal Lagrange multiplier is given by $\lambda^*(x)=-\log(z^*(x))\geq 0$.

Combining the above results, the optimal policy of the entropy-regularized problem is therefore given by
\begin{equation}\label{eq:pi_softmax}
\pi^*_{\tau}(\cdot|x)\propto\exp\left(-\frac{Q(x,\!\cdot)+\lambda^*(x) \cdot Q_L(x,\!\cdot)}{\tau}\right).
\end{equation}

\subsection{Case 2: Continuous Action Space}\label{appendix:inner_opt_cont}
In order to effectively solve the inner optimization problem in \eqref{eq:approx_greedy_pol} when the action space is continuous, 
on top of the using the entropy-regularized inner optimization problem in \eqref{eq:approx_greedy_pol_entropy}, we adopt the idea from the normalized advantage functions (NAF) approach for function approximation, where we express the $Q-$function and the state-action Lyapunov function with their second order Taylor-series expansions at an arbitrary action $\mu(x)$ as follows:
\[
\begin{split}
Q(x,a)\approx& Q(x,\mu(x))+\nabla_a Q(x,a)|_{a=\mu(x)} \cdot(a-\mu(x))\\
&+\frac{1}{2}
(a - \mu(x))^\top \nabla^2_a Q(x,a)|_{a=\mu(x)}(a -\mu(x))+o(\|a-\mu(x)\|^3),\\
Q_L(x,a)\approx& Q_L(x,\mu(x))+\nabla_a Q_L(x,a)|_{a=\mu(x)} \cdot(a-\mu(x))\\
&+\frac{1}{2}
(a - \mu(x))^\top \nabla^2_a Q_L(x,a)|_{a=\mu(x)}(a -\mu(x))+o(\|a-\mu(x)\|^3).
\end{split}
\]
While these representations are more restrictive than the general function approximations, they provide a receipe to determine
the policy, which is a minimizer of problem \eqref{eq:approx_greedy_pol} analytically for the updates in the safe AVI and safe DQN algorithms. In particular, using the above parameterizations, notice that
\[
\begin{split}
&\text{logsumexp}\left(-\frac{Q(x,\!\cdot)+\lambda^*(x) \cdot Q_L(x,\!\cdot)}{\tau}\right)\\
=&-\frac{1}{2}\log\left|A_{\lambda^*}(x)\right|-\frac{1}{2}\psi_{\lambda^*}(x)^\top A^{-1}_{\lambda^*}(x)\psi_{\lambda^*}(x)+K_{\lambda^*}(x)+\frac{n}{2}\log(2\pi),
\end{split}
\]
where $n$ is the dimension of actions,
\[
\begin{split}
A_{\lambda^*}(x)=&\frac{1}{\tau}\left(\nabla^2_a Q(x,a)|_{a=\mu(x)}+\lambda^*(x)\nabla^2_a Q_L(x,a)|_{a=\mu(x)}\right),\\\psi_{\lambda^*}(x)=&-\frac{1}{\tau}\left(\nabla_a Q(x,a)|_{a=\mu(x)}+\lambda^*(x)\nabla_a Q_L(x,a)|_{a=\mu(x)}\right)-A_{\lambda^*}(x)\mu(x),
\end{split}
\]
and 
\[
\begin{split}
K_{\lambda^*}(x)=&-\frac{Q(x,\mu(x))+\lambda^*(x)Q_L(x,\mu(x))}{\tau}\\
&+\left(\frac{1}{\tau}\left(\nabla_a Q(x,a)|_{a=\mu(x)}+\lambda^*(x)\nabla_a Q_L(x,a)|_{a=\mu(x)}\right)^\top-\frac{1}{2}\mu(x)^\top A_{\lambda^*}(x)\right)\mu(x)
\end{split}
\]
is a normalizing constant (that is independent of $a$). Therefore, according to the closed-form solution of the policy in \eqref{eq:pi_softmax},  
the optimal policy of problem \eqref{eq:approx_greedy_pol_entropy} follows a Gaussian  distribution, which is given by 
\[
\pi^*_{\tau}(\cdot|x)\sim\mathcal N(A_{\lambda^*}(x)^{-1}\psi_{\lambda^*}(x), A_{\lambda^*}(x)^{-1}).
\] 
In order to completely characterize the solution policy, it is still required to compute the Lagrange multiplier $\lambda^*(x)$, which is a polynomial root solution of \eqref{eq:lambda}. Since the action space is continuous, one can only approximate the integral (over actions) in this expression with numerical integration techniques, such as Gaussian quadrature, Simpson's method, or Trapezoidal rule etc. (Notice that if $\pi_B$ is a Gaussian policy, there is a tractable closed form expression for $\pi_B(\cdot|x)^\top Q_L(x, \cdot)$.)

\newpage
\section{Experimental Setup}\label{appendix:experiment_setup}
In the CMDP planning experiment, in order to demonstrate the numerical efficiency of the safe DP algorithms, we run a larger example that has a grid size of $60\times 60$.
To compute the LP policy optimization step, we use the open-source SciPy $\mathrm{linprog}$ solver. In terms of the computation time, on average every policy optimization iteration (over all states) in SPI and SVI takes approximately $25.0$ seconds, and for this problem SVI takes around $200$ iterations to converge, while SPI takes $60$ iterations. On the other hand the Dual LP method
computes an optimal solution, its computation time is over $9500$ seconds.

In the RL experiments, we use the Adam optimizer with learning rate $0.0001$.  
At each iteration, we collect an episode of experience (100 steps) and perform
10 training steps on batches of size 128 sampled uniformly from the replay buffer.
We update the target Q networks every 10 iterations and the baseline policy every 50 iterations.

For discrete observations, we use a feed-forward neural network with hidden layers of size 16, 64, 32, and relu activations.

For image observations, we use a convolutional neural network with filters of size $3\times3\times3\times32$, $32\times3\times3\times64$, and $64\times3\times3\times128$, with $2\times2$ max-pooling and relu activations after each.  We then pass the result through a 2-hidden layer network with sizes 512 and 128.

\begin{figure}[h]
\begin{center}
  \begin{tabular}{ccc}
    \includegraphics[width=0.25\columnwidth, angle =270]{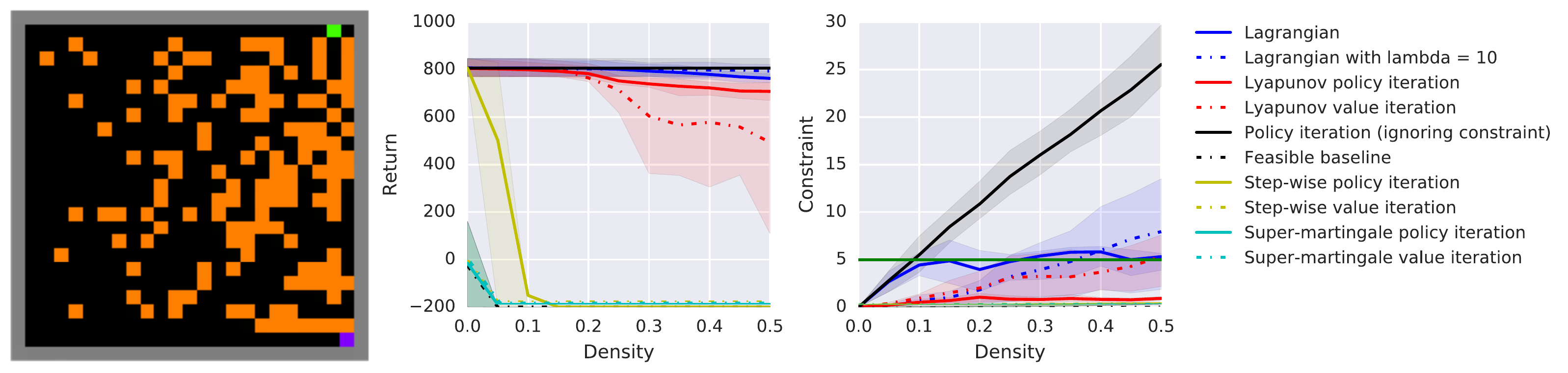} &
  \end{tabular}
\end{center}
\caption{
Results of various planning algorithms on the grid-world environment with obstacles (zoomed), with x-axis showing the obstacle density. From the leftmost column, the first figure illustrates the 2D planning domain example. The second and the third figure show the average return and the average cumulative constraint cost of the CMDP methods respectively. The fourth figure displays all the methods used in the experiment. The shaded regions indicate the $80\%$ confidence intervals. Clearly the safe DP algorithms compute policies that are safe and have good performance.
}
\label{fig:results1}
\end{figure}

\begin{figure}[h]
\begin{center}
  \begin{tabular}{ccc}
    & \small Discrete obs, $d_0=5$ & \small Discrete obs, $d_0=1$ \\
    \multirow{2}{*}{\rotatebox[origin=c]{90}{\tiny Constraints \hspace{2.1cm} Rewards\hspace{-1.7cm}}} &
    \includegraphics[width=0.2\columnwidth]{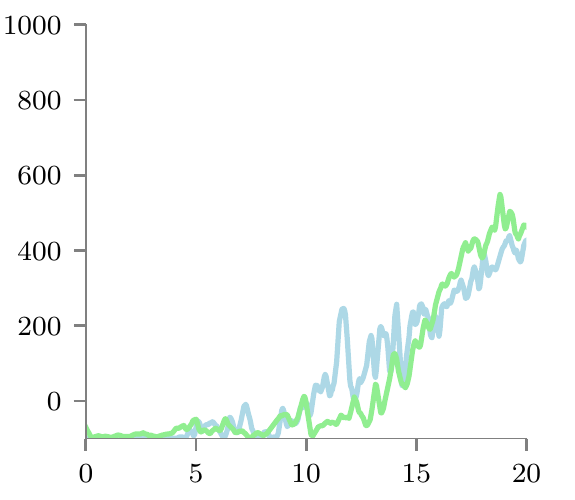} &
    \includegraphics[width=0.2\columnwidth]{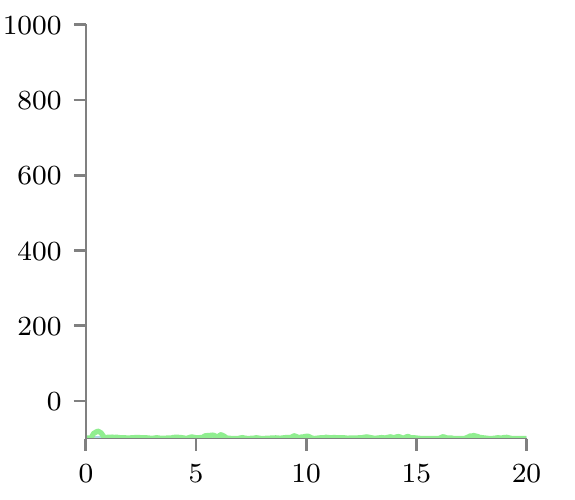} \\
    &
    \includegraphics[width=0.2\columnwidth]{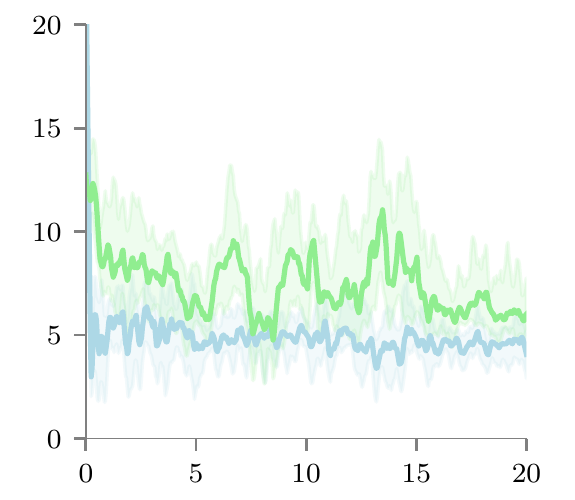} &
    \includegraphics[width=0.2\columnwidth]{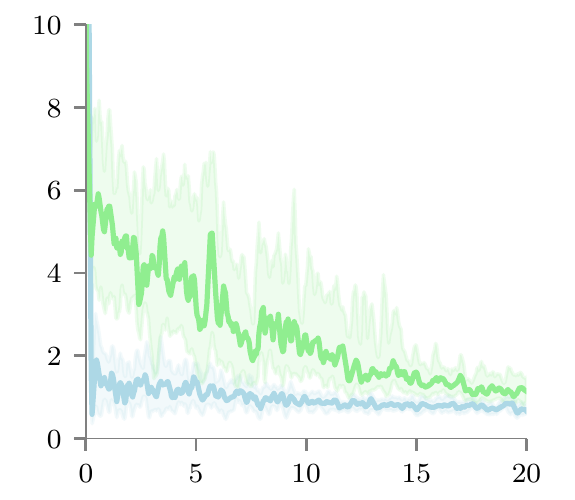} \\
    \multicolumn{3}{c}{\includegraphics[width=0.6\columnwidth]{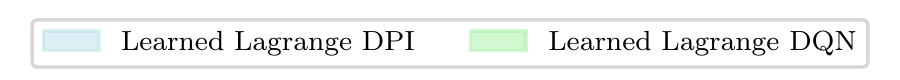}}
  \end{tabular}
\end{center}
\caption{
Results of using a saddle-point Lagrangian optimization for solving the grid-world environment with obstacles, with x-axis in thousands of episodes.  
}
\label{fig:results3}
\end{figure}
\end{document}